\documentclass[nohyperref]{article}

\usepackage{microtype}
\usepackage{graphicx}
\usepackage{subfigure}
\usepackage{booktabs} 
\usepackage{multicol}
\usepackage{hyperref}

\usepackage[accepted]{icml2023}

\usepackage{amsmath}
\usepackage{amssymb}
\usepackage{mathtools}
\usepackage{amsthm}
\usepackage{MnSymbol}
\usepackage{bm}
\usepackage[capitalize,noabbrev]{cleveref}
\usepackage{wrapfig}
\theoremstyle{plain}
\newtheorem{theorem}{Theorem}[section]

\newtheorem{lemma}[theorem]{Lemma}
\newtheorem{corollary}[theorem]{Corollary}
\theoremstyle{definition}
\newtheorem{definition}[theorem]{Definition}

\theoremstyle{remark}

\usepackage[textsize=tiny]{todonotes}

\usepackage{xcolor}
\usepackage{multirow}
\usepackage{xr}
\usepackage{subfiles}
\usepackage{makecell}
\definecolor{mydarkblue}{rgb}{0,0.08,0.45}
\definecolor{mydarkgreen}{RGB}{0, 139, 69}
\definecolor{mygreen2}{RGB}{0 205 0}
\definecolor{mybrown}{RGB}{205 133 63}
\hypersetup{
 colorlinks=true,
 urlcolor=magenta,
 citecolor=mydarkgreen,
}
\definecolor{mycyan}{cmyk}{.3,0,0,0}

\icmltitlerunning{Beware of Instantaneous Dependence in Reinforcement Learning}

\begin{document}

\twocolumn[
\icmltitle{Beware of Instantaneous Dependence in Reinforcement Learning}




\icmlsetsymbol{equal}{*}

\begin{icmlauthorlist}
\icmlauthor{Zhengmao Zhu}{nju}
\icmlauthor{Yuren Liu}{nju}
\icmlauthor{Honglong Tian}{nju}
\icmlauthor{Yang Yu}{nju}
\icmlauthor{Kun Zhang}{cmu,mbzuai}
\end{icmlauthorlist}

\icmlaffiliation{nju}{National Key Laboratory for Novel Software Technology, Nanjing University, Nanjing, Jiangsu, China}
\icmlaffiliation{cmu}{Department of Philosophy, Carnegie Mellon University, Pittsburgh, PA, United States}
\icmlaffiliation{mbzuai}{Mohamed bin Zayed University of Artificial Intelligence, Abu Dhabi, United Arab Emirates}

\icmlcorrespondingauthor{Kun Zhang}{kunz1@cmu.edu}

\icmlkeywords{Machine Learning, ICML}

\vskip 0.3in
]



\printAffiliationsAndNotice{}  

\begin{abstract}

Playing an important role in Model-Based Reinforcement Learning (MBRL), environment models aim to predict future states based on the past. 
Existing works usually ignore instantaneous dependence in the state, that is, assuming that the future state variables are conditionally independent given the past states.
However, instantaneous dependence is prevalent in many RL environments.
For instance, in the stock market, instantaneous dependence can exist between two stocks because the fluctuation of one stock can quickly affect the other and the resolution of price change is lower than that of the effect.
In this paper, we prove that with few exceptions, ignoring instantaneous dependence can result in suboptimal policy learning in MBRL.
To address the suboptimality problem, we propose a simple plug-and-play method to enable existing MBRL algorithms to take instantaneous dependence into account. 
Through experiments on two benchmarks, we (1) confirm the existence of instantaneous dependence with visualization; (2) validate our theoretical findings that ignoring instantaneous dependence leads to suboptimal policy;
(3) verify that our method effectively enables reinforcement learning with instantaneous dependence and improves policy performance.

\end{abstract}

\section{Introduction}
Reinforcement Learning (RL) is a general framework to tackle sequential decision-making problems through trials and errors, which can be categorized into Model-Free methods and Model-Based methods \cite{sutton-introduction}.
Model-Free Reinforcement Learning (MFRL) learns a direct mapping from states to actions while Model-Based Reinforcement Learning (MBRL) builds an environment model and derives a policy from it.
In real-world physical systems, where trials and errors can be costly, MBRL is a more appealing approach due to its promise of data efficiency \cite{Kaelbling-rl-survey}.
Meanwhile, the performance of MBRL heavily depends on the accuracy of the environment model \cite{luo-survey,moerland2023model}.

To achieve an accurate environment model, existing methods attempt to introduce various priors into environment model networks or design different optimization objects for model learning.  
In introducing priors, the Gaussian Process is used to model the stochasticity of transitions \cite{gaussian-model,gaussian-mixture-model} and causal structures are introduced to model the data generation process \cite{causal-structure-2,causal-structure-1};
In designing optimization objects, the prediction loss is extended from one-step to multi-step \cite{multi-step-2,multi-step-1} and from forward-direction to bi-direction \cite{bidirection}.
To simplify environment modeling, the aforementioned methods mainly focus on how to model the transition between states and explicitly assume that the variables within states are mutually independent given the past state.

However, this assumption is too strong to be satisfied in many cases, as the dependence within states, known as instantaneous dependence, is widespread in RL environments. 
For instance, in the stock market, instantaneous dependence can exist between two stocks because the fluctuation of one stock can quickly affect the other and the resolution of the price change is lower than that of the effect.
Similarly, in robotic control, instantaneous dependence can exist between two joints because the disturbance of one joint can quickly affect the other.

In this paper, we investigate when disregarding instantaneous dependence hurts policy learning in MBRL. 
We introduce the concept of a ``lagged'' model, which neglects instantaneous dependence, and an ``instantaneous'' model, which takes it into account.
We formalize the MBRL process that learns and leverages a lagged model in an environment with instantaneous dependence. 
Through our formalization, we prove that the optimal policy obtained by a lagged model is strictly inferior to that obtained in the true environment, with only few exceptions when the reward function and the transition function have specific simple forms.
Specifically, for most transition functions, we can always find a family of reward functions such that the optimal policy of the true environment cannot be learned from a lagged model.
To address the issue of suboptimality, we propose a simple plug-and-play method to enable existing MBRL algorithms to consider instantaneous dependence. 
To simulate states with instantaneous dependence, we utilize non-diagonal covariance matrices in the probabilistic environment model, where the covariance matrices are learned by leveraging prediction errors.

To validate our theoretical results and highlight the advantages of our proposed approach over baselines, we conduct experiments on two benchmarks.
Experiments on dynamics-aware control tasks verify our theoretical findings that ignoring instantaneous dependence in MBRL does not hurt policy learning only in exceptional instances.
Experiments on Mujoco, a complex and dynamics-unaware environment, reveal that our approach greatly outperforms baselines with lower model prediction error and higher policy return.

In summary, this paper presents three key contributions:
\begin{itemize}
    \item To the best of our knowledge, this is the first work to consider instantaneous dependence in RL from the policy learning perspective.  
    It theoretically demonstrates that disregarding instantaneous dependence leads to suboptimal policies. 
    \item It proposes a simple plug-and-play method to enable existing MBRL algorithms to account for instantaneous dependence, which facilitates policy learning in stochastic environments with unknown instantaneous dependence.
    \item Our experimental results verify our theoretical findings and demonstrate that in contexts with instantaneous dependence, our proposed approach outperforms the baseline using a lagged model.
\end{itemize}

\section{Related Work}
\subsection{Model-Based Reinforcement Learning}
MBRL algorithms learn virtual transition models by fitting them with the trajectory-based data collected by running the latest policy \cite{DBLP:conf/icml/DeisenrothR11, DBLP:conf/nips/BellemareSOSSM16, DBLP:conf/nips/ChuaCML18, DBLP:conf/iclr/LuoXLTDM19}, which tends to guarantee that the virtual model is accurate when the policy keeps exploring \cite{kaiser2019model,moerland2023model}.
One widely used approach to modeling the transition is to directly minimize the distance between the predictions and the ground truth data, where the distance can be a mean square error or likelihood probability with various re-weighting methods \cite{mbpo, mopo}.
To improve the model accuracy, existing works attempt to introduce prior knowledge into the model or derive different optimization objects for model learning.
In introducing priors,\citet{gaussian-model} and \citet{gaussian-mixture-model} use the Gaussian process and Gaussian mixture process to model the stochasticity in the transition, respectively.
\citet{causal-structure-2} introduces causal structures inferred with conditional mutual information while \citet{causal-structure-1} considers offline RL and introduces causal structures inferred with kernel-based independence test.
In deriving optimization objects, \citet{multi-step-1,multi-step-2,multi-step-3} consider model errors based on multiple-step predictions, which benefits decreasing compounding errors.
\citet{bidirection} considers model errors based on forward-direction and backward-direction predictions, which reduces the reliance on accuracy in forward-direction model predictions.
Another approach tries to leverage the value information for model learning, called Decision-Aware Model Learning \cite{DBLP:conf/aistats/FarahmandBN17, DBLP:conf/nips/Farahmand18}.  
The aforementioned methods learn the stochastic transition function by minimizing its KL divergence with the posterior state distribution, where both prior and posterior distributions are assumed to be the Gaussian with diagonal covariance \cite{hafner2019learning, hafner2020dream, hafner2021mastering}. 
The main difference between this paper and existing works is that we consider instantaneous dependence and can predict states with non-diagonal covariance.

\subsection{Instantaneous dependence for Time-Series Data}
In causal modeling, instantaneous causal influence has been noticed and studied together with time-delayed ones at the same time. 
\citet{swanson1997impulse} proposes a method for testing structural models of the errors in vector autoregressions, which first studies instantaneous dependence in noises.
Past works mainly focus on estimating the causal effects or discovering instantaneous causal relations from the instantaneous dependence in observations.
For example, \citet{DBLP:journals/jmlr/HyvarinenZSH10} proposes a non-Gaussian model to estimate both instantaneous and lagged effects.
The instantaneous dependence in the estimated noise terms might also be caused by the lower resolutions of the measured data, compared to the original causal frequency.
\citet{Temporally-Aggregated} studies the instantaneous dependence from temporally aggregated data, which are obtained by taking the average (or sum) of every non-overlapping $k$ time step.
With aggregation, independent noises are mixed up to form new noises, which are not mutually independent.
\citet{subsampled-data} studies the instantaneous dependence from subsampled data, which are collected every $k$ time step.
Similarly, with subsampling, the new noises contain multi-step noises, which are not mutually independent either.
The two works both focus on identifying the one-step transition from observed data.

The main difference between this paper and previous works is that we study the effect of instantaneous dependence on policy performance in decision-making scenarios while they focus on causal modeling and the identifiability considering instantaneous dependence.

\section{Preliminaries}
\subsection{Markov Decision Process (MDP).}

We describe the RL environment as an MDP with five-tuple $\langle \mathcal{S}, \mathcal{A}, P, R, \gamma \rangle$ \cite{bellman1957markovian}, where $\mathcal{S}$ is the state space with dimension $n$;
$\mathcal{A}$ is the action space; 
$P$ is the transition function with $P(\textbf{s}'|\textbf{s}, \textbf{a})$ denoting the next-state distribution after taking action $\textbf{a}$ in state $\textbf{s}$; 
$R$ is a reward function with $R(\textbf{s}, \textbf{a})$ denoting the expected immediate reward gained by taking action \textbf{a} in state \textbf{s}; 
$\gamma \in [0, 1]$ is a discount factor. 
An agent chooses actions $\textbf{a}$ according to a policy $\textbf{a}\sim \pi(\textbf{s})$, which updates the system state $\textbf{s}'\sim P(\textbf{s},\textbf{a})$, yielding a reward $r\sim R(\textbf{s},\textbf{a})$. 
The agent's goal is to find the optimal policy $\pi^*$ that maximizes the expected sum of discounted rewards, denoted by $J(\pi)$:
\begin{align*}
    \pi^*=\arg\max_\pi J(\pi)=\arg\max_{\pi} \mathbb{E}_{\textbf{s}_t\sim P,\textbf{a}_t\sim\pi}[\gamma^t R(\textbf{s}_t,\textbf{a}_t)].
\end{align*}
The transition function $P(\textbf{s}'|\textbf{s}, \textbf{a})$ is assumed to be unknown.
MBRL methods aim to learn an environment model of the transition function by using data from interactions with the MDP.
We also introduce the definitions of value functions, which are involved in our theoretical analysis.
The state-action value function $Q_\pi$ of a policy $\pi$ is the expected discounted reward of executing action $\textbf{a}$ from state $\textbf{s}$ and subsequently following policy $\pi$:
$Q_{\pi}(\textbf{s}, \textbf{a}):=R(\textbf{s}, \textbf{a})+\gamma \mathbb{E}_{\textbf{s}' \sim P, \textbf{a}' \sim \pi}\left[Q_{\pi}\left(\textbf{s}', \textbf{a}'\right)\right]$.
Similarly, the state value function $V_\pi$ is the expected discounted reward from state $\textbf{s}$ and subsequently following policy $\pi$:
$V_{\pi}(\textbf{s}):=\mathbb{E}_{\textbf{a} \sim \pi}\left[Q_{\pi}\left(\textbf{s}, \textbf{a}\right)\right]$.

\subsection{Instantaneous Dependence}

For an RL environment with instantaneous dependence, its transition function can be described in the following form:
\begin{align*}
    s_{t+1,i}=[P(\textbf{s}_{t},\textbf{a}_{t},e_{t+1,i})]_i,
\end{align*}
where $[\cdot]_i$ represents the $i^{th}$ dimension of a vector, $e_{t+1,i}$ represents the random noise in the generation process of $x_{t+1,i}$ and $\textbf{e}_{t+1}=(e_{t+1,i})_{i=1}^n$ are not mutually independent to each other.
As mentioned in related work, the dependence in $\textbf{e}_{t}$ can result from subsampling or data aggregation.
For example, consider an environment with a transition function $\textbf{s}_{t+1}=A\textbf{s}_{t}+\epsilon_{t+1}$ where $A$ is the transition matrix and $\epsilon_{t+1}$ are independent noises.
When data is subsampled that the observations are $\tilde{\textbf{s}}_{t}=\textbf{s}_{1+(t-1)k}$, we have
\begin{align*}
    \tilde{\textbf{s}}_{t+1}=A^k\tilde{\textbf{s}}_{t}+\tilde{\textbf{e}}_{t}
    &=A^k\tilde{\textbf{s}}_{t}+\sum_{l=0}^{k-1}A^l \epsilon_{1+tk-l},
\end{align*}
where $\tilde{\textbf{e}}_{t}$ are not mutually independent.
Similarly, when data is aggregated that $\tilde{\textbf{s}}_{t}=\frac{1}{k}\sum_{i=1}^k \textbf{s}_{i+(t-1)k}$, we have
\begin{align*}
    \tilde{\textbf{s}}_{t+1}=A^k\tilde{\textbf{s}}_{t}+\tilde{\textbf{e}}_{t}
    =&A^k\tilde{\textbf{s}}_{t}+\frac{1}{k}(\sum_{m=0}^{k-1}(\sum_{n=0}^{m}A^n)\epsilon_{tk-m}\\
    &+\sum_{m=1}^{k-1}(\sum_{n=m}^{k-1}A^n)\epsilon_{(t-1)k-m+1}),
\end{align*}
where $\tilde{\textbf{e}}_{t}$ are also not mutually independent.

\section{Theory}

In this section, we investigate when disregarding instantaneous dependence hurts policy learning in MBRL.
To facilitate clear discussion, we call a virtual model \textbf{optimality-consistent} if the optimal policy from the virtual model is the same as that in the true environment.
Since interactions with the virtual model are convenient, we can easily obtain its optimal policy and thus we focus on whether disregarding instantaneous dependence hurts optimality-consistency.

Using this concept, we identify situations in which optimality-consistent is always satisfied and situations in which it can never be satisfied.
We provide proofs organized according to the following logic.
In the first subsection ``Problem Formulation'', we formalize the concept of optimality-consistency.
And in the second subsection ``The properties of a lagged model'', we derive Theorem~\ref{theorem-GV*}, which is crucial for finding the conditions for optimality-consistency.
In the consequent third subsection ``The conditions for optimality-consistency'', we prove Theorem~\ref{multi}, which concludes the sufficient conditions for optimality-consistency.
Finally in the fourth subsection ``The conditions for optimality-inconsistency'', we leverage Theorem~\ref{multi} to prove that for certain situations, we can always find a reward function that causes a lagged model to be optimality-inconsistent, meaning that the optimal policy from the lagged model is not optimal in the true environment.

\subsection{Problem Formulation}

Given a stochastic RL environment $\mathcal{M}=\langle \mathcal{S}, \mathcal{A}, P, R, \gamma \rangle$ defined as preliminaries, we denote $Q^*_{P},V^*_{P}$ as the optimal value functions and $\pi^*_{P}$ as the optimal policy in the true transition $P$ and correspondingly $Q^*_{\hat{P}},V^*_{\hat{P}}, \pi^*_{\hat{P}}$ for the virtual transition model $\hat{P}$.
The virtual transition model aims to achieve optimality-consistency, that is, the policy $\pi^*_{\hat{P}}$ has the same policy return as $\pi^*_{P}$:
\begin{align*}
    J(\pi^*_{P})=J(\pi^*_{\hat{P}}).
\end{align*}
In contrast, optimality-inconsistency corresponds to the situation that the optimal policy from the learned transition model has a smaller policy return, that is, $J(\pi^*_{P})>J(\pi^*_{\hat{P}})$.
We expand the policy return as defined in the preliminaries and have the following lemma:
\begin{lemma}
\label{lemma-Q*}
    Given an arbitrary stochastic transition function $P$ and corresponding $\hat{P}$, if there exists $\tilde{\textbf{s}}$ that is visited by the optimal policy $\pi^*_{P}$, such that 
    \begin{align*}
        &Q^*_{\hat{P}}(\tilde{\textbf{s}},\pi^*_{P}(\tilde{\textbf{s}}))<Q^*_{\hat{P}}(\tilde{\textbf{s}},\pi^*_{\hat{P}}(\tilde{\textbf{s}}))\\
        &Q^*_{P}(\tilde{\textbf{s}},\pi^*_{P}(\tilde{\textbf{s}}))>Q^*_{P}(\tilde{\textbf{s}},\pi^*_{\hat{P}}(\tilde{\textbf{s}})),
    \end{align*}
    then we have that the policy return of $J(\pi^*_{P})>J(\pi^*_{\hat{P}})$, that is, $\hat{P}$ is optimality-inconsistent.
\end{lemma}

\subsection{The properties of a lagged model}

In this subsection, we formalize the learning process of using a lagged model in an environment with instantaneous dependence and derive Theorem~\ref{theorem-GV*}, which is the base for the following theorems.

We denote $s_{t,i}$ as the $i^{th}$ dimension of $\textbf{s}_t$.
For a lagged transition model $\hat{P}$ learned with the independent noise assumption, it satisfies that for $i,j=1, 2, \cdots \ n$,
\begin{align*}
    \hat{P}(s_{t,i},s_{t,j}|\textbf{s}_{t-1},a_{t-1})
    &= \hat{P}(s_{t,i}|\textbf{s}_{t-1},a_{t-1})\hat{P}(s_{t,j}|\textbf{s}_{t-1},a_{t-1})\\
    \hat{P}(s_{t,i}|\textbf{s}_{t-1},a_{t-1})&=P(s_{t,i}|\textbf{s}_{t-1},a_{t-1})
\end{align*}
And for any instantaneous dependence between $(s_{t,i},s_{t,j})$ in the true transition $P$, we have
\begin{align*}
    \hat{P}(s_{t,i},s_{t,j}|\textbf{s}_{t-1},a_{t-1})\neq P(s_{t,i},s_{t,j}|\textbf{s}_{t-1},a_{t-1}).
\end{align*}
The proof is attached in the appendix.

Based on the properties, we have the following corollary:
\begin{corollary}
\label{corollary-integral}
For any stochastic transition function $P$ and corresponding $\hat{P}$, any $(\textbf{s}_{t-1},a_{t-1})$, we have four corollaries as follows:

(1) For any constant $C$, we have
\begin{align*}
    \int [P(\textbf{s}_t|\textbf{s}_{t-1},a_{t-1})-\hat{P}(\textbf{s}_t|\textbf{s}_{t-1},a_{t-1})]Cd\textbf{s}_t=0.
\end{align*}
(2) For any $s_{t,i}\in\textbf{s}_t$ and an arbitrary function $f(s_{t,i})$, we have
\begin{align*}
    &\int [P(\textbf{s}_t|\textbf{s}_{t-1},a_k)-\hat{P}(\textbf{s}_t|\textbf{s}_{t-1},a_k)]f(s_{t,i})d\textbf{s}_t=0.   
\end{align*}
(3) For any independent variable pair $(s_{t,i},s_{t,j})$ and an arbitrary function $g(s_{t,i},s_{t,j})$, we have
\begin{align*}
    &\int [P(\textbf{s}_t|\textbf{s}_{t-1},a_k)-\hat{P}(\textbf{s}_t|\textbf{s}_{t-1},a_k)]g(s_{t,i},s_{t,j})d\textbf{s}_t=0.
\end{align*}
(4) If there exists instantaneous dependence $(s_{t,i},s_{t,j})$ in $P$, then there exists a function $g(s_{t,i},s_{t,j})$ that cannot be divided into $g_1(s_{t,i})+g_2(s_{t,j})$ such that
\begin{align*}
    &\int [P(\textbf{s}_t|\textbf{s}_{t-1},a_k)-\hat{P}(\textbf{s}_t|\textbf{s}_{t-1},a_k)]g(s_{t,i},s_{t,j})d\textbf{s}_t\neq 0.
\end{align*}
\end{corollary}

The above corollary shows that a lagged model $\hat{P}$ is equivalent to the true transition $P$ when evaluated in the form of the integral with functions belonging to $(1)\sim (3)$.
Inspired by above equivalence, we decompose an arbitrary function $F(\textbf{s}_t)$ into four components as follows:
\begin{align}
\label{formula-V-partition}
    F(\textbf{s}_t) = C+\sum_{i=1}^{n_f} f_i(s_{t,i})+\sum_{k=1}^{n_g}g_k(\textbf{s}_{t,D_k})+\sum_{k=n_g+1}^{n_g+n_h} g_k(\textbf{s}_{t,D_k}),
\end{align}
where $C$ is a constant, $f_i(s_{t,i})$ represents the terms that only depend on one dimension of $\textbf{s}_t$, $\textbf{s}_{t,D_k}$ represents the dimensions of those are in $D_k\subset \{1,2,\cdots, n\}$ and $|D_k|>1$, $g_k$ is a function that cannot be divided into the form of $g_{k,1}(\textbf{s}_{t,D_k}/\textbf{s}_{t,D'_k})+g_{k,2}(\textbf{s}_{t,D'_k})$ and $|D'_k|\geq1$.
For $1\leq k\leq n_g$, there exists at least one instantaneous-dependent pair $(s_{t,i},s_{t,j})$ in $\textbf{s}_{t,D_k}$.
And for $n_g+1\leq k\leq n_g+n_h$, there do not exist such pairs in $\textbf{s}_{t,D_k}$.

With the decomposition, we denote that 
$$G_{F}(\textbf{s}_t) = \sum_{k=1}^{n_g}g_k(\textbf{s}_{t,D_k}),$$ 
and reveal a more general equivalence between the lagged virtual model $\hat{P}$ and the true model $P$, which is represented in the following theorem:
\begin{theorem}
    \label{theorem-GV*}
    For any $(\textbf{s}_{t-1},a_{t-1})$, any deterministic function $F(\textbf{s}_t)$ with the partition in Formula~\ref{formula-V-partition} and corresponding $G_{F}(\textbf{s}_t)$, any stochastic function $P$ and corresponding $\hat{P}$, we have
    \begin{align}
    \label{formula-GV*}
         &\int [P(\textbf{s}_t|\textbf{s}_{t-1},a_{t-1})-\hat{P}(\textbf{s}_t|\textbf{s}_{t-1},a_{t-1})]F(\textbf{s}_t) d\textbf{s}_t\nonumber \\  
    =&\int [P(\textbf{s}_t|\textbf{s}_{t-1},a_{t-1})-\hat{P}(\textbf{s}_t|\textbf{s}_{t-1},a_{t-1})]G_{F}(\textbf{s}_t) d\textbf{s}_t.
    \end{align}    
\end{theorem}
The proof can be easily finished by decomposing the function $F$ and leveraging Corollary~\ref{corollary-integral}.
We will use the above theorem to find sufficient conditions when a lagged model achieves optimality-consistency in the following subsection.

\subsection{The conditions for optimality-consistency}

In this subsection, we further derive a sufficient condition of optimality-consistency by leveraging Theorem~\ref{theorem-GV*} and the definition of the optimal value function.

Recall that optimality-consistency is not satisfied if 
\begin{align*}
        &Q^*_{\hat{P}}(\tilde{\textbf{s}},\pi^*_{P}(\tilde{\textbf{s}}))<Q^*_{\hat{P}}(\tilde{\textbf{s}},\pi^*_{\hat{P}}(\tilde{\textbf{s}}))\\
        &Q^*_{P}(\tilde{\textbf{s}},\pi^*_{P}(\tilde{\textbf{s}}))>Q^*_{P}(\tilde{\textbf{s}},\pi^*_{\hat{P}}(\tilde{\textbf{s}})).
\end{align*}
We denote $a_0=\pi^*_{P}(\tilde{\textbf{s}}), a_1=\pi^*_{\hat{P}}(\tilde{\textbf{s}})$ and
\begin{align*}
        \alpha=&Q^*_{P}(\tilde{\textbf{s}},a_0)-Q^*_{P}(\tilde{\textbf{s}},a_1)\\
        \beta=&Q^*_{\hat{P}}(\tilde{\textbf{s}},a_0)-Q^*_{\hat{P}}(\tilde{\textbf{s}},a_1)-\alpha.
\end{align*}
The condition can be simplified to
\begin{align*}
    0<\alpha<-\beta.
\end{align*}
And we can easily find that if $\beta\geq 0$ for any $(\textbf{s}_t,a_0,a_1)$, then there is no solution of $\alpha$, which means that $J(\pi^*_{P})=J(\pi^*_{\hat{P}})$ and the lagged model is optimality-consistent.

In the following, we will introduce when we can guarantee $\beta\geq 0$ for any $(\textbf{s}_t,a_0,a_1)$.
We first introduce a property for the transition function:
\begin{definition}
    We say a stochastic transition function $P(\textbf{s}_{t+1}|\textbf{s}_t,a_t)$ is $G-invariant$ 
    if for any function $F_1(\textbf{s}_{t+1})$ that $G_{F_1}(\textbf{s}_{t+1})\equiv 0$ and any policy $\pi(\textbf{s}_t)$, we have 
    \begin{align*}
        F_2(\textbf{s}_t)=\int P(\textbf{s}_{t+1}|\textbf{s}_t,\pi(\textbf{s}_t))F_1(\textbf{s}_{t+1})d\textbf{s}_{t+1},
    \end{align*}
    also satisfies that $G_{F_2}(\textbf{s}_t)\equiv 0$.
\end{definition}

With the above definition, we can prove that the $P$ and $\hat{P}$ share the same $Q$ values given certain conditions, which is the base to prove $\beta=0$: 
\begin{lemma}
\label{lemma-multi-k}
    If the reward function $R$ satisfies $G_R(\textbf{s}_{t})\equiv 0$, and the transition function $P$ is $G-invariant$, then $\forall t\in\mathbb{N}^+$, we have
    \begin{align*}
        Q^*_{P}(\textbf{s}_{t},a_{t})=Q^*_{\hat{P}}(\textbf{s}_{t},a_{t})
    \end{align*}
\end{lemma}

\textbf{Proof sketch.}
We prove it by induction.
First, we prove that for $i=1$, we have:
\begin{align*}
 Q^*_{P}(\textbf{s}_{T-i},a) =& Q^*_{\hat{P}}(\textbf{s}_{T-i},a), \ \forall a\\
\pi^*_{P}(\textbf{s}_{T-i})=&\pi^*_{\hat{P}}(\textbf{s}_{T-i})\\
V^*_{P}(\textbf{s}_{T-i}) =&V^*_{\hat{P}}(\textbf{s}_{T-i})\\
G_{V^*_{P}}(\textbf{s}_{T-i})\equiv& 0,
\end{align*}
which is obvious because $G_R(\textbf{s}_{T-1})\equiv 0$.
And then we prove that when the above equations hold true for $i=k$, the same equations for $i=k+1$ also hold true due to the transition function $P$ being $G-invariant$.

With the above lemma, we can easily prove that $\beta=0$ given certain conditions:
\begin{theorem}
\label{multi}
     In a RL environment, if for any policy $\pi$, the stochastic transition function $P(\textbf{s}_t|\textbf{s}_{t-1},\pi(\textbf{s}_{t-1}))$ is $G-invariant$ and the reward function $R(\textbf{s}_{t},\pi(\textbf{s}_{t}))$ as a function of $\textbf{s}_{t}$ satisfies that $G_R(\textbf{s}_{t})\equiv 0$, then we have $\beta = 0$ for any $(\textbf{s}_t,a_0,a_1)$.
\end{theorem}

Based on Theorem~\ref{multi}, when the transition function and the reward function satisfy the above conditions, we can guarantee that the policy from a lagged model is equivalent to the optimal policy.

In fact, it is difficult to satisfy either of the conditions in environments.
First, if the reward function is related to the action, then it is easy to find a policy $\pi$ such that $G_R(\textbf{s}_t)\nequiv 0$.
Second, even if the reward function is unrelated to the action, common reward functions like 
\begin{align*}
    &R(\textbf{s}_t)=\|\textbf{s}_t\|_2, \\
    &R(\textbf{s}_t)=\left\{
    \begin{array}{cc}
        1 ,& \min_i(s_{t,i}) > \textrm{threshold} \\
        0 ,& \min_i(s_{t,i}) \leq \textrm{threshold} 
    \end{array}
    \right.
\end{align*}
also satisfy that $G_R(\textbf{s}_t)\nequiv 0$.
Similarly, even if the transition function is unrelated to the action, $G-invariant$ is also a strict condition for $P$.
For example, if $s_{t,i}$ is the only causal parent of $s_{t+1,i}$ in $P$, then we can say $P$ is $G-invariant$.

\subsection{The conditions for optimality-inconsistency}

When the sufficient conditions for optimality-consistency are not satisfied, we cannot directly conclude that a lagged model is optimality-inconsistent.
Therefore we further investigate when optimality-inconsistency can be guaranteed in this subsection.
Inspired by the above theorem, we can prove that under some conditions, there always exists a family of reward functions that causes a lagged model to be optimality-inconsistent. 

We define a set of reward functions $\mathcal{D}_R$, which comprises reward functions that causes optimality-inconsistency:
\begin{definition}
Given the transition function $P$ and a lagged model $\hat{P}$, the state and actions $(\textbf{s}_{t-1},a_0,a_1)$, whether $0<\alpha<-\beta$ holds true depends on the reward function $R$.
Therefore we define a set of reward functions $\mathcal{D}_R$ that:
\begin{align*}
    \mathcal{D}_R=\{R | 0<\alpha(R)<-\beta(R)\}.
\end{align*}
\end{definition}
We aim to prove that $\mathcal{D}_R \neq \emptyset$ for any $\textbf{s}_{t-1},a_0, a_1$ under some conditions.
Here we provide two theorems, which guarantee $\mathcal{D}_R \neq \emptyset$ under two different cases.

\textbf{Case 1: Single-step RL.}

In this case, we consider a simple RL environment where the trajectory only contains a single transition step.
We first provide a lemma to show key properties of $\beta(R)$:
\begin{lemma}
    In single-step RL, for any stochastic transition function $P$, any $(\textbf{s}_{t-1},a_0,a_1)$, any reward function $R$, we have
    \begin{align}
        &\beta(R)=-\beta(-R).
        \label{formual-beta-r}
    \end{align}
    And for any function $f$ and state $s_{t,i}$, we have 
    \begin{align}
        &\beta(R(\textbf{s}_{t-1},a_{t-1},\textbf{s}_t))=\beta(R(\textbf{s}_{t-1},a_{t-1},\textbf{s}_t)+f(s_{t,i})).
        \label{formual-beta-f}
    \end{align}
\end{lemma}
Based on Formula~\ref{formual-beta-r}, we assure that there exists a reward function to make $\beta<0$.
And based on Formula~\ref{formual-beta-f}, we can keep $\beta$ invariant but $\alpha$ changed by adding $f(s_{t,i})$ to the reward function.
Therefore, it is easy to design a reward function to make $0<\alpha<-\beta$, as we claim in the following:
\begin{theorem}
\label{single-DR}
    In single-step RL, for any stochastic transition function $P$ with instantaneous dependence, any $(\textbf{s}_{t-1},a_0,a_1)$, we have
    $\mathcal{D}_R \neq \emptyset$.
\end{theorem}

\textbf{Case 2: The transition function $P$ is $G-invariant$.}

In this case, we extend environments from single-step to multi-step while imposing a condition that the transition function $P$ is $G-invariant$.
With this condition, we strictly prove that $\mathcal{D}_R \neq \emptyset$:
\begin{theorem}
\label{multi-DR}
    For any stochastic and $G-variant$ transition function $P$ with instantaneous dependence, any $(\textbf{s}_{t-1},a_0,a_1)$, we have
    $\mathcal{D}_R \neq \emptyset$.
\end{theorem}
Similarly, we can add $f(s_{t,i})$ to $R(\textbf{s}_{t-1},a_t,\textbf{s}_t)$ with which $\beta$ is invariant and $\alpha$ changes.
For situations where the transition function $P$ is not $G-invariant$, we provide empirical evidence in the experiment section.

\section{Algorithm}

In this section, we present an MBRL algorithm designed for environments with instantaneous dependence.
Our algorithm builds on MBPO \cite{mbpo}, which is a Dyna-style approach that learns policy with real data and simulated data.
For a stochastic RL environment, existing environmental models learn the mean and variance of the next time-step state for each dimension and make predictions by sampling independent random variables, e.g., standard normal random variables \cite{}.
These predictions assume that the next time-step state variables are conditionally independent given the past state.
To incorporate instantaneous dependence in model prediction, we introduce non-diagonal covariance matrix predictions. 
Specifically, we denote the correlation coefficient matrix as $\Gamma$ and the resulting model rollout process is presented in Algorithm~\ref{algo-model}.

\begin{algorithm}[tb]
   \caption{Model Rollout With Instantaneous dependence}
   \label{algo-model}
\begin{algorithmic}
   \STATE {\bfseries Input:} empty dataset $\mathcal{D}$; start state $\textbf{s}_t$; rollout step $k$; rollout policy $\pi$; rollout environmental model $p_\theta$; Correlation coefficient matrix $\Gamma$.
   \FOR {$\textrm{i}=0$ to $k$}
   \STATE {Sample action $a_{t+i}\sim \pi(\textbf{s}_{t+i})$}
   \STATE {Sample noise $\textbf{e}$ from multivariate normal distribution: $\textbf{e}\sim \mathcal{N}_{Multi}(\bm{0},\Gamma)$}
   \STATE {Get mean and variance for prediction: $\mu,\Sigma=p_\theta(\textbf{s}_{t+i},a_{t+i})$}
   \STATE {Make prediction $\textbf{s}_{t+i+1}=\bm{\mu}+\textbf{e}^T\Sigma^{\frac{1}{2}}$}
    \STATE Record data: $\mathcal{D}\gets \mathcal{D}\cup\{(\textbf{s}_{t+i},\textbf{a}_{t+i},\textbf{s}_{t+i+1})\}$
\ENDFOR
\STATE {Return dataset $\mathcal{D}$}
\end{algorithmic}
\end{algorithm}

Based on Algorithm~\ref{algo-model}, we design an MBRL algorithm directly computing instantaneous dependence by prediction errors as shown in Algorithm~\ref{algo-all}.

\begin{algorithm}[tb]
   \caption{Model-Based Policy Optimization Learning With Instantaneous dependence}
   \label{algo-all}
\begin{algorithmic}
   \STATE {\bfseries Input:} Initialize policy $\pi_\phi$, predictive model $p_\theta$, environment dataset $\mathcal{D}_{env}$, model dataset $\mathcal{D}_{model}$, Initial unit correlation coefficient matrix $\Gamma$.
   \FOR {$N$ epochs}
   \STATE {Train model $p_\theta$ on $\mathcal{D}_{env}$ via maximum likelihood}
   \FOR {$E$ steps}
   \STATE {Take action in environment according to $\pi_\phi$; add to $\mathcal{D}_{env}$}
   \STATE {Compute prediction errors based on unit correlation matrix $\Gamma_0=\textbf{I}$, $\mathcal{D}_{env}$ and $p_\theta$:}
   \STATE {$\textbf{e}_t = \textbf{s}_{t+1}-\hat{\textbf{s}}_{t+1}, \hat{\textbf{s}}_{t+1}\sim (\Gamma_0,p_\theta)$}
   \STATE {Update $\Gamma$ based on prediction error $\textbf{e}_t$}
   \FOR {$M$ model rollouts}
   \STATE {Sample $s_t$ uniformly from $\mathcal{D}_{env}$}
   \STATE {Perform $k$-step rollout through Algorithm~\ref{algo-model}; add to $\mathcal{D}_{model}$}
   \ENDFOR
   \FOR {$G$ gradient updates}
    \STATE {Update policy parameters on model data: $\phi\gets\phi-\lambda_\pi \nabla_\phi J_\pi(\phi,\mathcal{D}_{model})$}

\ENDFOR
\ENDFOR
\ENDFOR
\end{algorithmic}
\end{algorithm}

\section{Experiments}

In this section, we aim to validate our theory with three parts:
(1) Validate the suboptimality of the policy learned from a lagged model by visualizing its difference from the optimal policy in the true environment.
(2) Validate our derived conditions for optimality-inconsistency by estimating policy returns under different environment settings.
(3) Validate that our proposed method successfully enables MBRL algorithms to consider instantaneous dependence and improves policy performance by estimating model prediction errors and policy returns.
In addition, we also prove that our algorithm is insensitive to our mentioned hyperparameters in the sensitivity study.

\textbf{Evaluation Metrics.}
(1) Likelihood probability. We compute the likelihood probabilities of the real data in a lagged model and in our model. 
Higher likelihood probabilities indicate that the real data is more likely to be sampled from the corresponding model, which suggests lower model prediction errors.
(2) Policy return. We compute the average returns in the true environment for the optimal policies from a lagged model and our model. 
Higher policy returns indicate that the corresponding model is better for policy learning and closer to the true environment.

\textbf{Baselines.}
We choose MBPO, a widely-used MBRL algorithm with asymptotic performance rivaling the best model-free algorithms, as the baseline.
In the following, we use ``INS'' to represent our proposed method that learns an instantaneous model while ``LAG'' represents the baseline that learns a lagged model in MBRL.

\textbf{Environment.}
\textbf{(1) 1-D Driving}.
In this environment, the state is composed of the position and velocity of a car, $(p_t,v_t)$, and the action $a_t$ represents the acceleration of the car. 
The goal is to approximate the origin point $(0,0)$ from random start positions.
We design this simple environment to validate the first part above, i.e., visualizing the differences between the two policies.
\textbf{(2) CartPole}.
CartPole \cite{gym} is a classic environment where a pole is attached by an unactuated joint to a cart that moves along a frictionless track.
This environment provides convenience for setting different reward functions and other parameters, allowing us to validate the second part above.
\textbf{(3) MuJoCo}. \begin{wrapfigure}{r}{0.5\linewidth}
    \centering
    \includegraphics[width=0.9\linewidth]{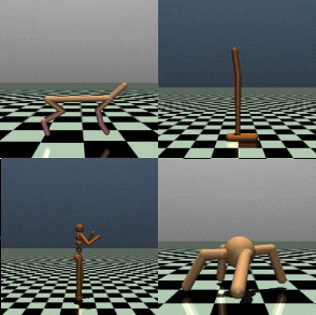}
    \caption{The MuJoCo environments.}
    \label{fig-mujoco-env}
\end{wrapfigure}
MuJoCo \citep{todorov2012mujoco} is a popular benchmark for evaluating MBRL algorithms, which is essentially a physical simulation system with states composed of positions, angles, and velocities.
The transition function in MuJoCo is set to be deterministic and in most cases, the reward is $1$ if the robot satisfies some poses and $0$ otherwise.
To transform it into a stochastic environment with instantaneous dependence, we add dependent noises for the angle and the velocity of each joint and reward items related to dependent state variables.
We validate the third part in this environment.

\subsection{Visualization of optimality-inconsistency}

To visualize the optimality-inconsistency, we plot the optimal policies from a lagged model and the true environment in the 1-D Driving environment.
We first provide a detailed description of this environment, in which the theoretical optimal policies can be easily obtained.
And then we perform simulations to validate the theoretical policies and plot them in figures, allowing us to easily identify the optimality-inconsistency from the two different figures.

In this environment, a car moves in a 1-D space with observations $\textbf{s}_t=(v_t,p_t)$ and actions $a_{t}$ where $p_t$ represents the position, $v_t$ represents the velocity and $a_{t}\subset\{A_0,A_1\}$ represents two different acceleration processes.
The transition function is that 
\begin{align*}
v_t&=v_{t-1}+(-\frac{p_{t-1}}{|p_{t-1}|})\Delta v(a_{t-1})+g(a_{t-1})\epsilon_v \\
    p_t&=p_{t-1}+v_t\Delta t +\epsilon_p,
\end{align*}
where $\Delta v(a_{t-1})$ represents the velocity change and $g(a_{t-1})$ represents the noise magnitude under acceleration $a_{t-1}$.
$(-\frac{p_{t-1}}{|p_{t-1}|})$ is the coefficient of $\Delta v(a_{t-1})$ to make sure that actions always help the position to approximate the origin point.  
We set the max length of a trajectory as $200$, and the start states $(p,v)$ are uniformly sampled from $[-2,2]^2$ and $\gamma=1$.
Here we add $-1$ penalty for each step and stop the game when the car reaches the area that $\sqrt{p^2+v^2}<0.1$.
Here we fill details of task parameters as follows:

\begin{table}[ht]
\begin{center}
\begin{small}
\begin{sc}
    \begin{tabular}{lccccc}
    \toprule
        name&$(\Delta v(A_i))_{i=0}^1$ & $\Delta t$ & $\sigma_v$ & $\frac{g(A_i)}{\Delta v(A_i)}$ & $V^*(p,v)$ \\
        \midrule
        value&$(0.1, 1)$ & $1$ & $1$ & $0.1$ & $-(|p|+|v|)^2$\\
\bottomrule
    \end{tabular}

    \caption{The parameters of 1-D Driving.}
    \label{table-visual-param}
    \end{sc}
\end{small}
\end{center}
\end{table}

\begin{table}[ht]
\begin{center}
\begin{small}
\begin{sc}
    \begin{tabular}{lccc}
    \toprule
         &Policy return & distance & length \\
        \midrule
        INS&$\bm{-57\pm67}$&$\bm{0.067\pm0.023}$ & $\bm{76.63\pm 66.51}$ \\
        \midrule
        LAG&$-143\pm128$& $0.099\pm0.221$ & $162.10\pm 128.20$\\
\bottomrule
    \end{tabular}
    \caption{The comparison of policy returns for using a lagged model and an instantaneous model.
    ``DISTANCE'' represents the distance between the car and the origin point when the game ends.
    ``LENGTH'' represents the steps taken to reach the target area.}
    \label{table-visual-score}
    \end{sc}
\end{small}
\end{center}
\vspace{-2mm}
\end{table}

\begin{figure}
    \centering
    \includegraphics[width=0.9\linewidth]{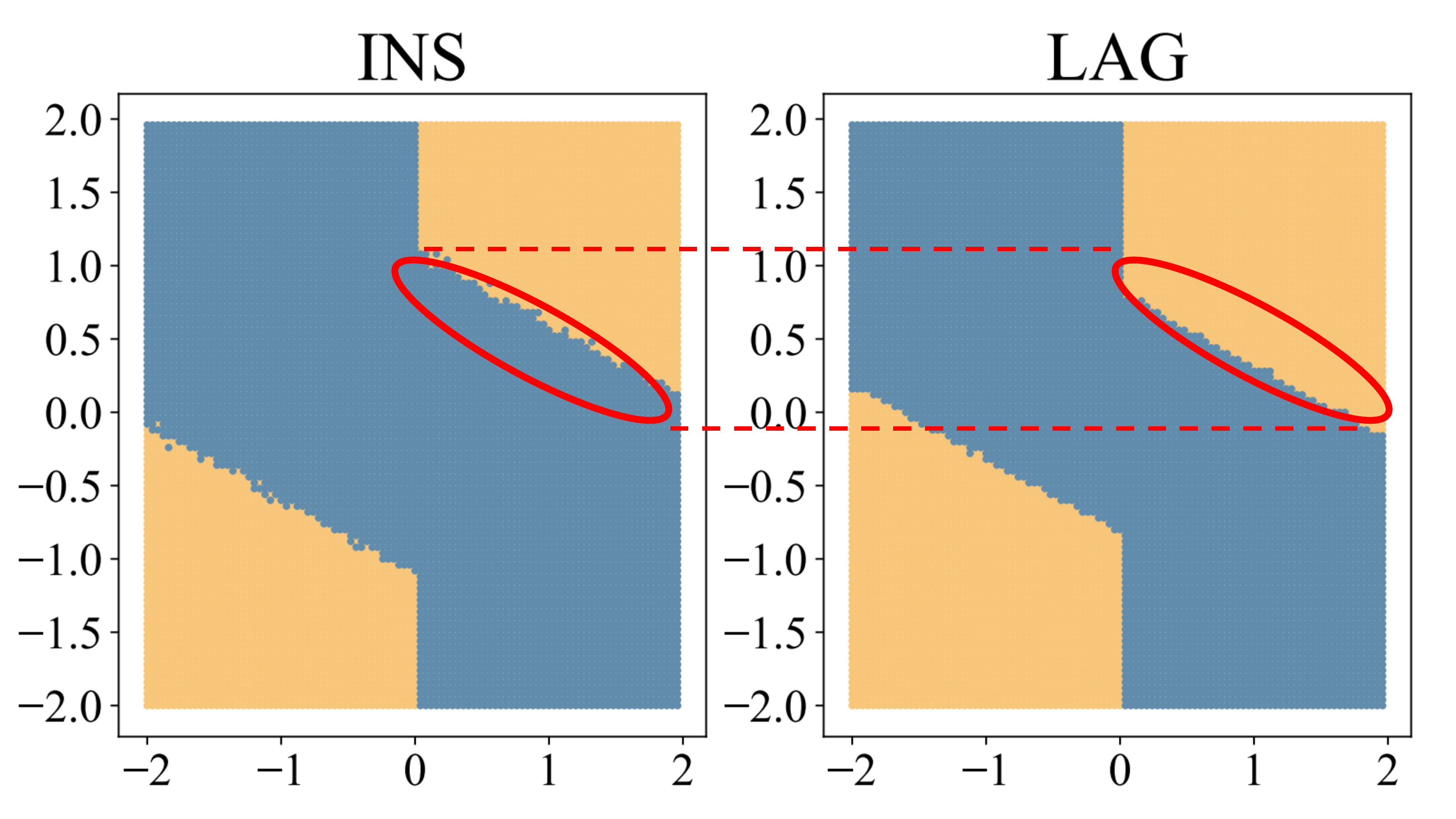}
    \caption{The visualization of the two policies, which is best viewed in color. 
    The coordinates represent the two dimensions in the state, i.e., $(p,v)$.
    Blue represents that the optimal policy chooses action $a_0$ at this state and yellow represents choosing $a_1$.
    \textbf{Left:} The optimal policy from the environment with instantaneous dependence. 
    \textbf{Right:} The optimal policy from a lagged model.
    We circle the difference between the two images.
    }
    \label{fig-visual}
    \vspace{-2mm}
\end{figure}

\begin{figure}[th!]
    \centering
    \includegraphics[width=0.9\linewidth]{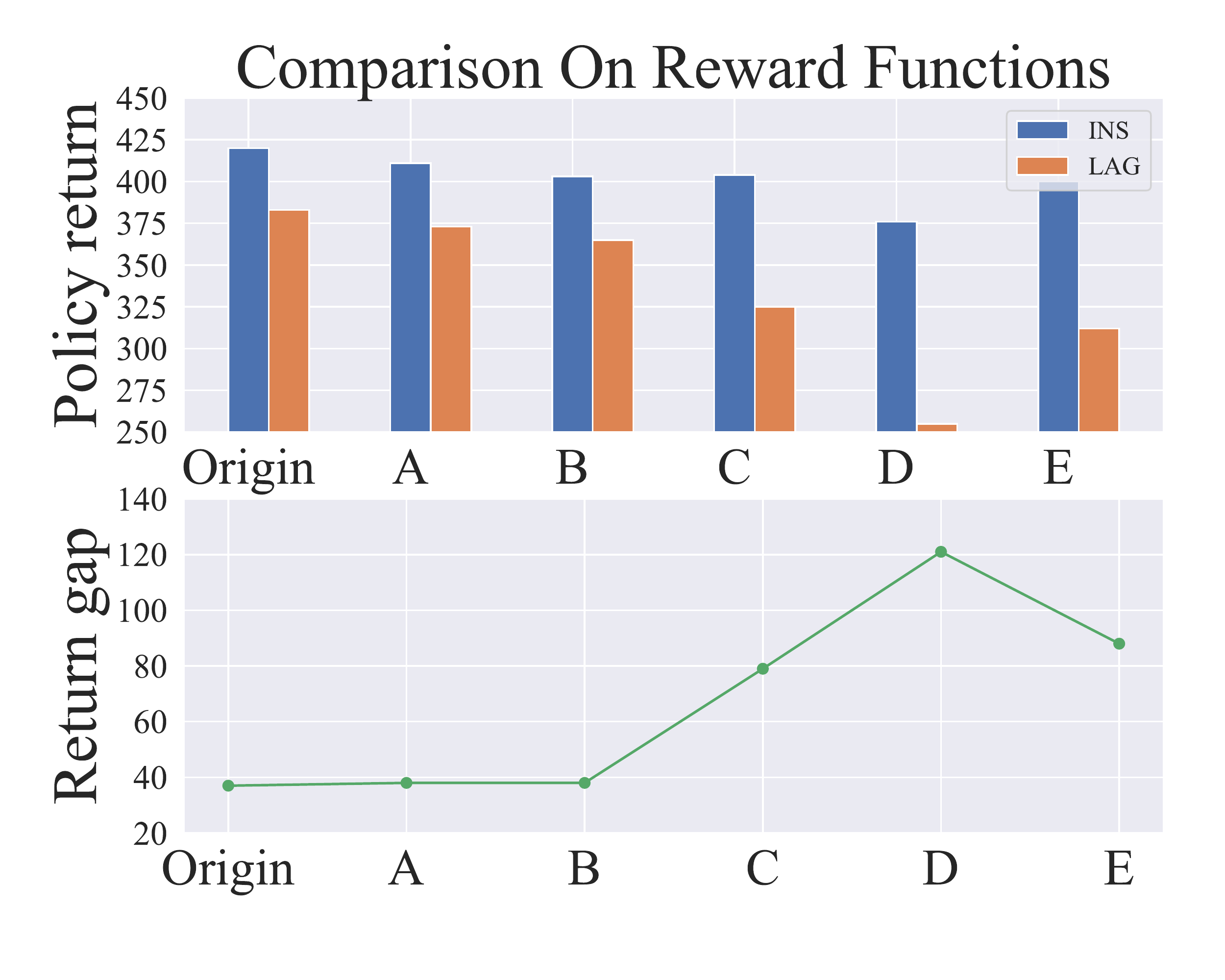}
    \caption{The comparison of INS model and LAG model on various reward functions.
    \textbf{Top:} The policy returns of the two models with the origin reward and five different rewards.
    \textbf{Bottom:} The gap of policy returns of the two models with different rewards.}
    \label{fig-reward-bar}
     \vspace{-2mm}
\end{figure}

\begin{figure*}[ht!]
\centering
\subfigure{
        \centering
        \includegraphics[width=0.22\linewidth]{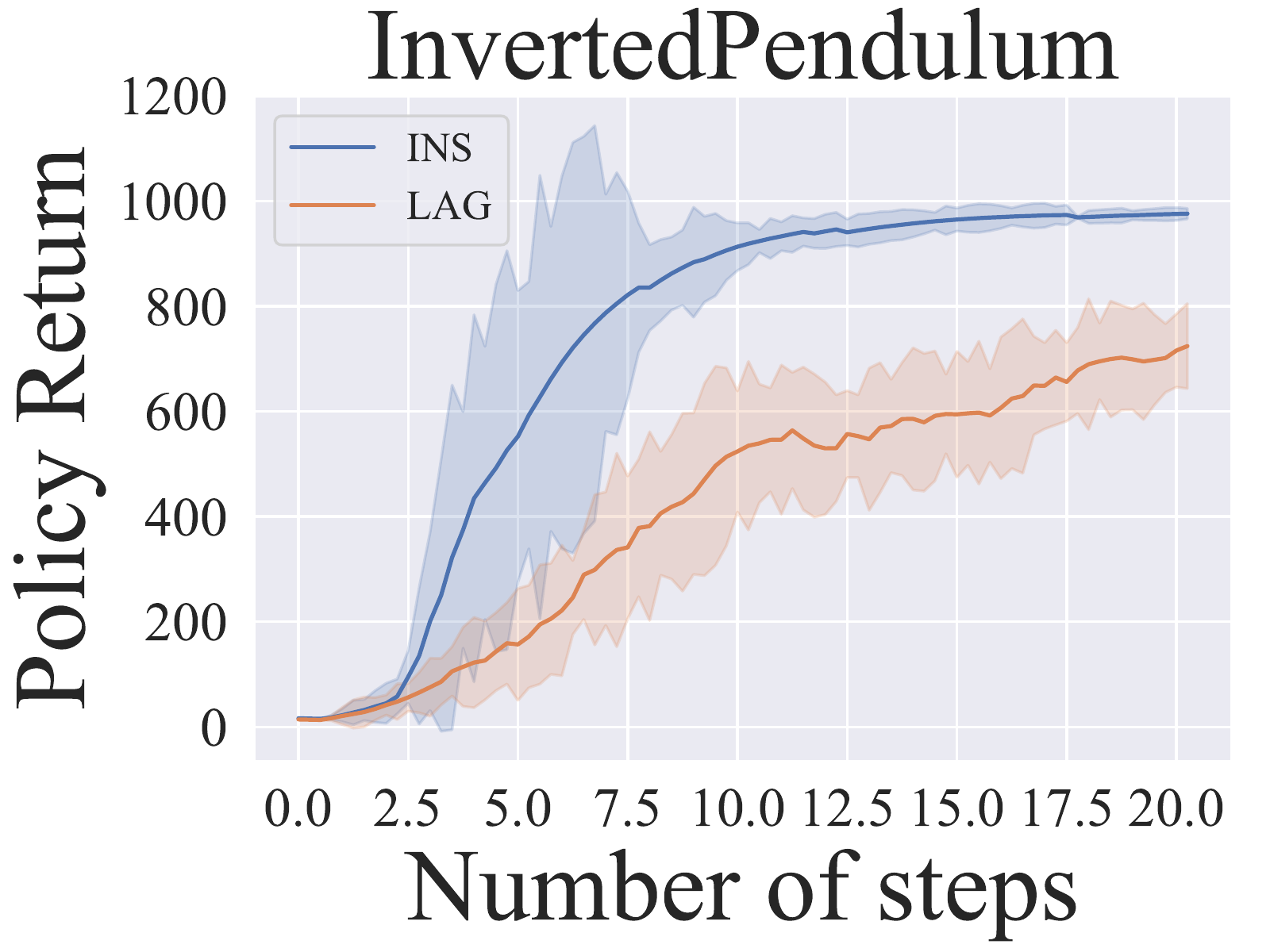} 
        \label{fig-return-inv}}
\subfigure{
        \centering
        \includegraphics[width=0.22\linewidth]{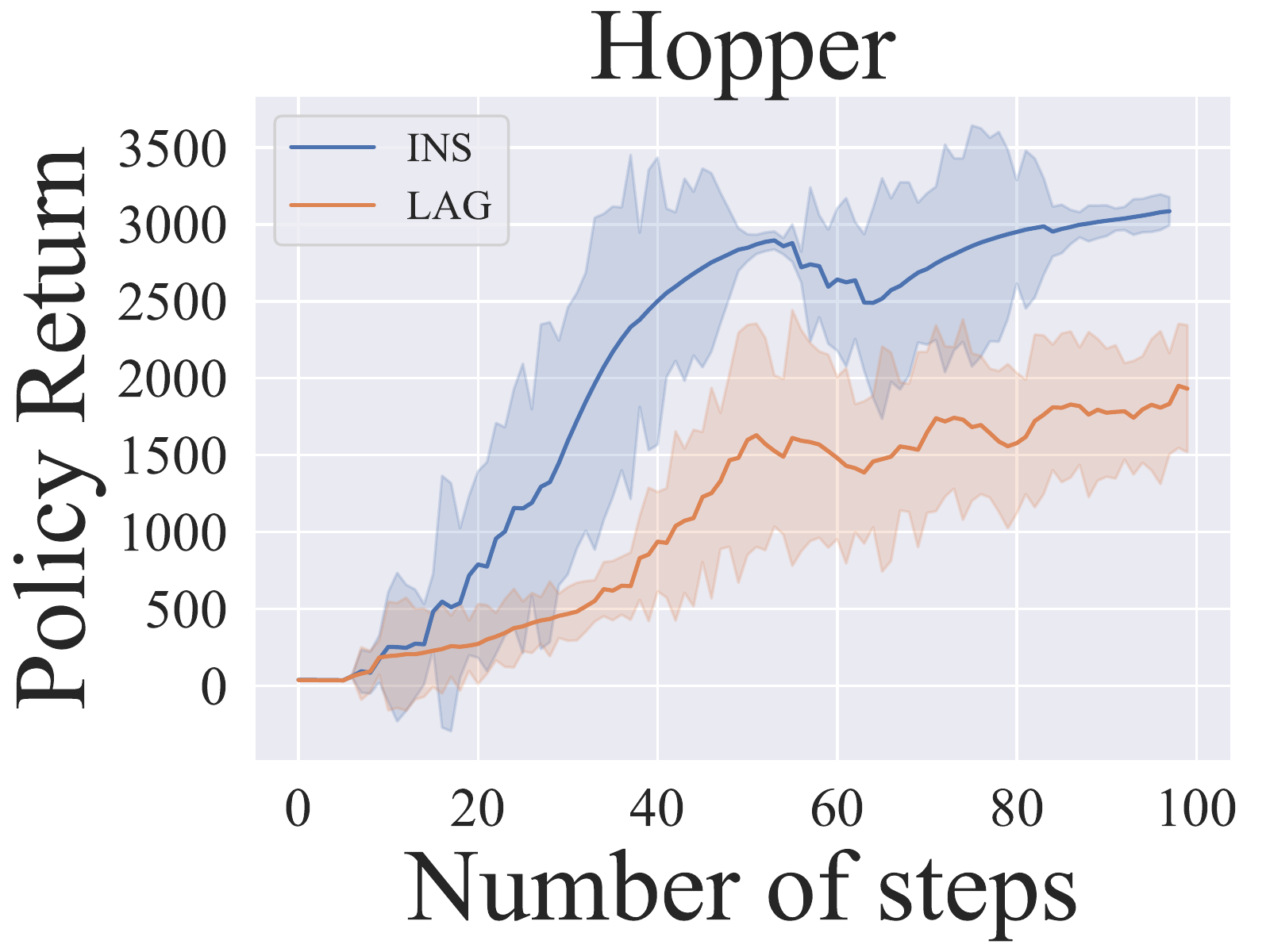} 
        \label{fig-return-inv}}
\subfigure{
        \centering
        \includegraphics[width=0.22\linewidth]{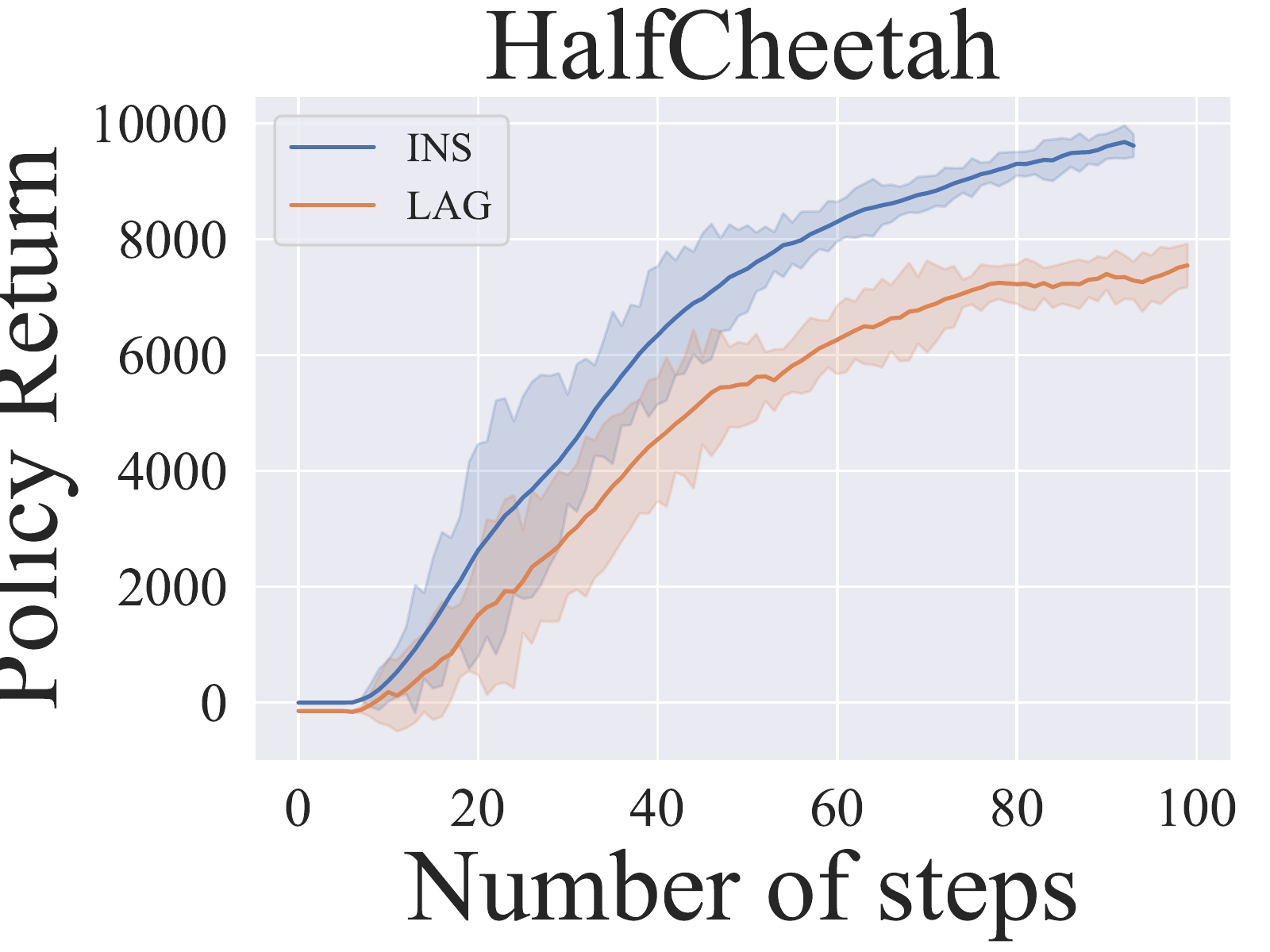} 
        \label{fig-return-inv}}
\subfigure{
        \centering
        \includegraphics[width=0.22\linewidth]{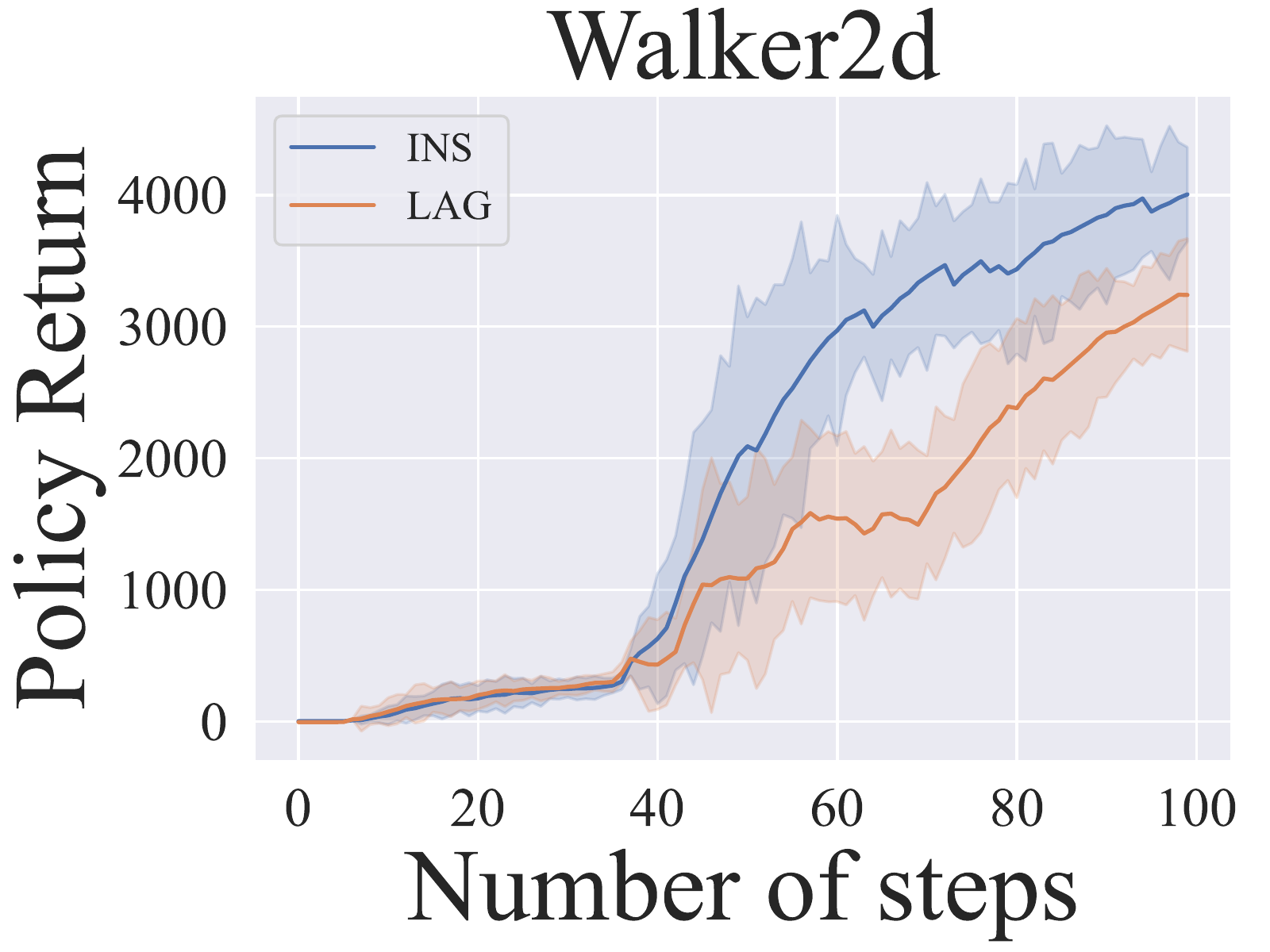} 
        \label{fig-return-inv}}

\subfigure{
        \centering
        \includegraphics[width=0.22\linewidth]{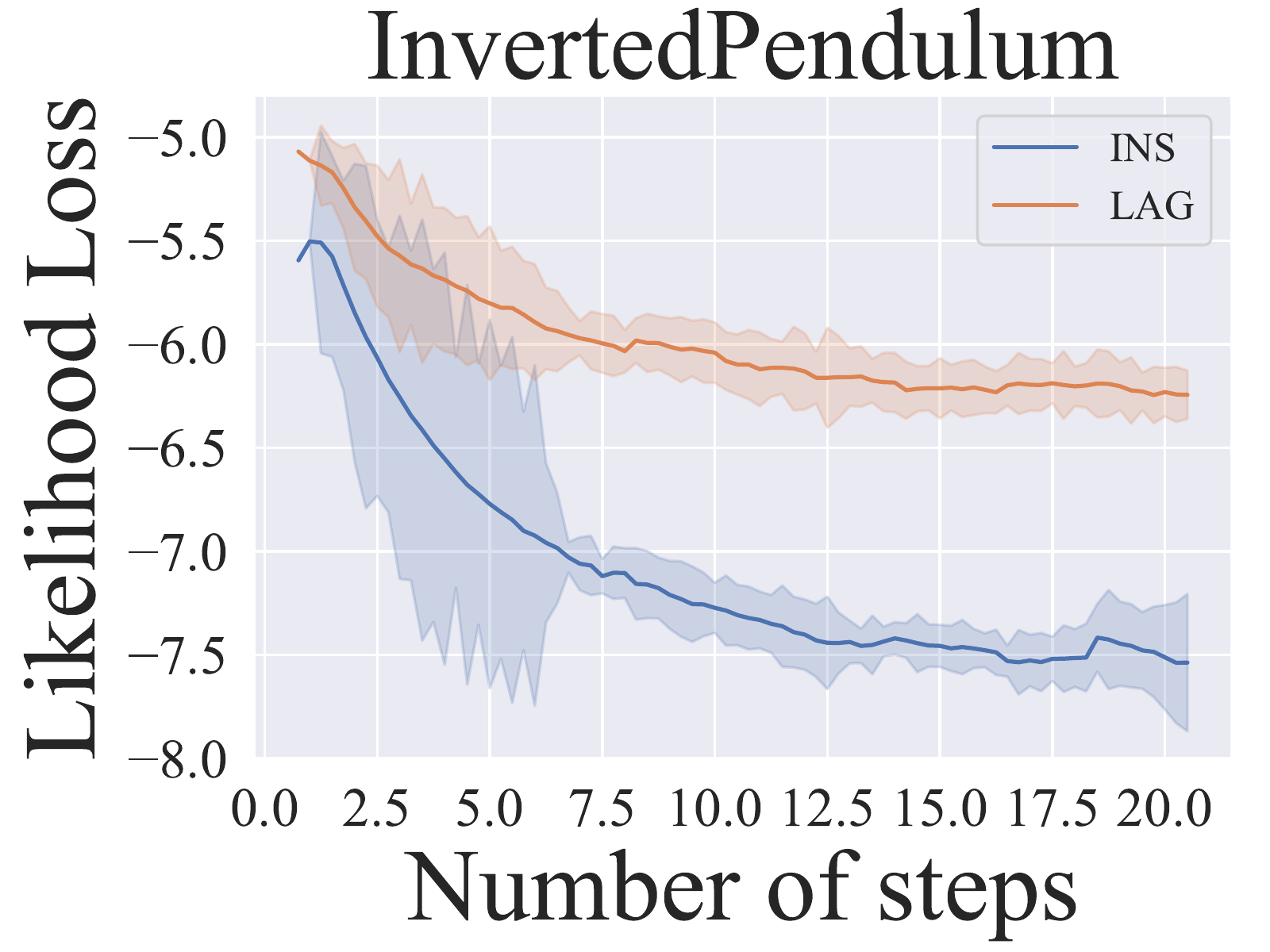} 
        \label{fig-return-inv}}
\subfigure{
        \centering
        \includegraphics[width=0.22\linewidth]{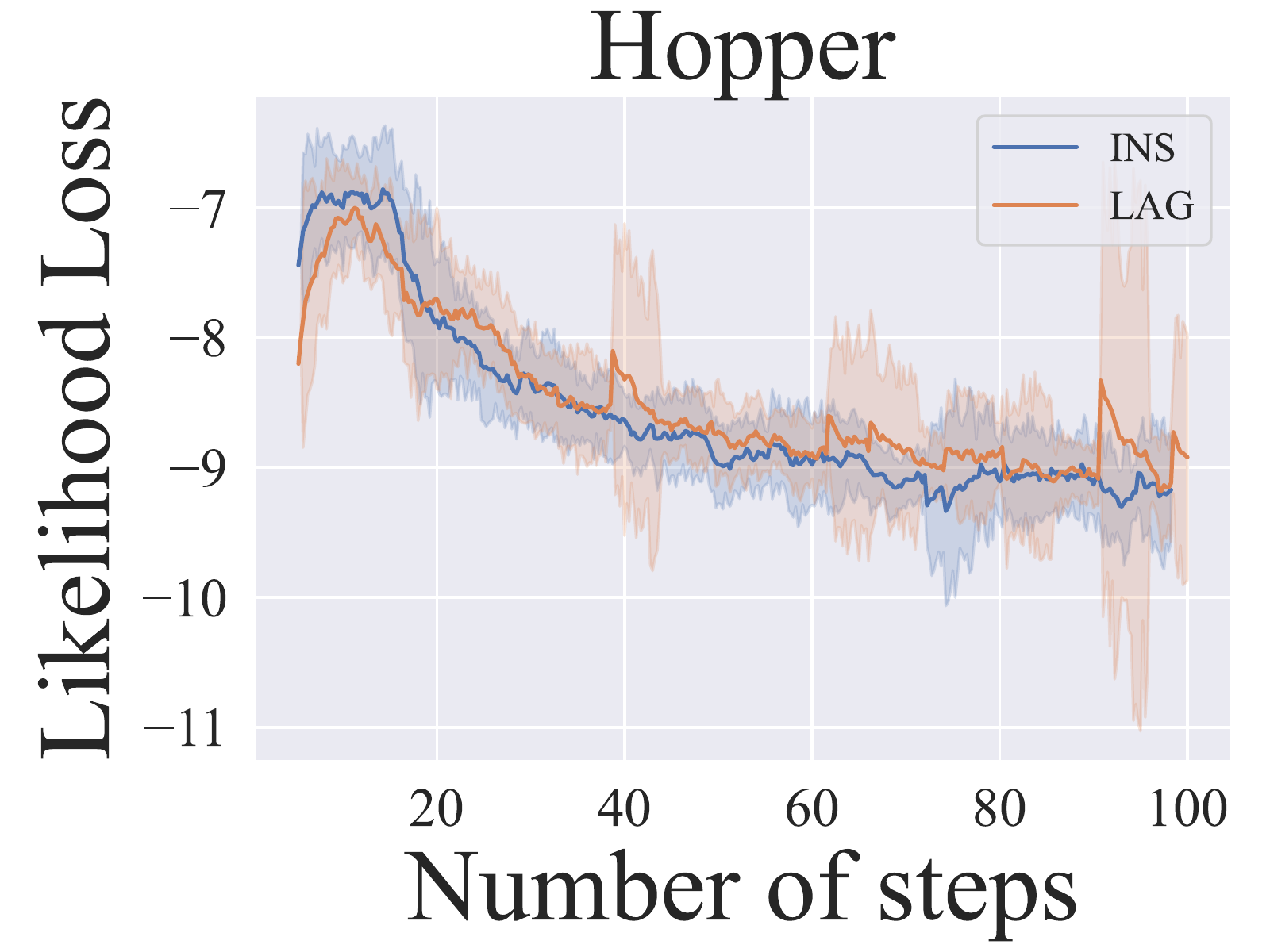} 
        \label{fig-return-inv}}
\subfigure{
        \centering
        \includegraphics[width=0.22\linewidth]{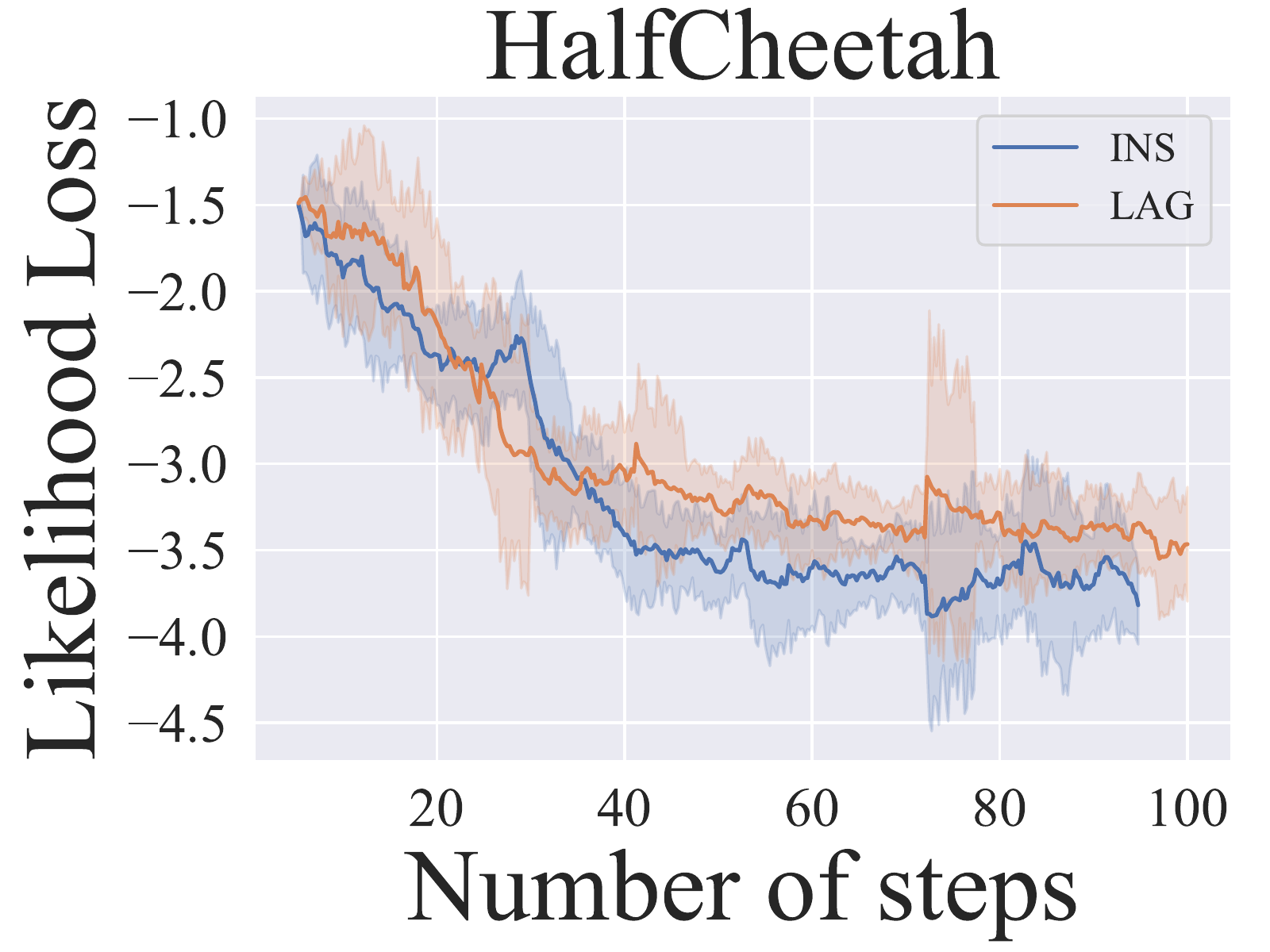} 
        \label{fig-return-inv}}
\subfigure{
        \centering
        \includegraphics[width=0.22\linewidth]{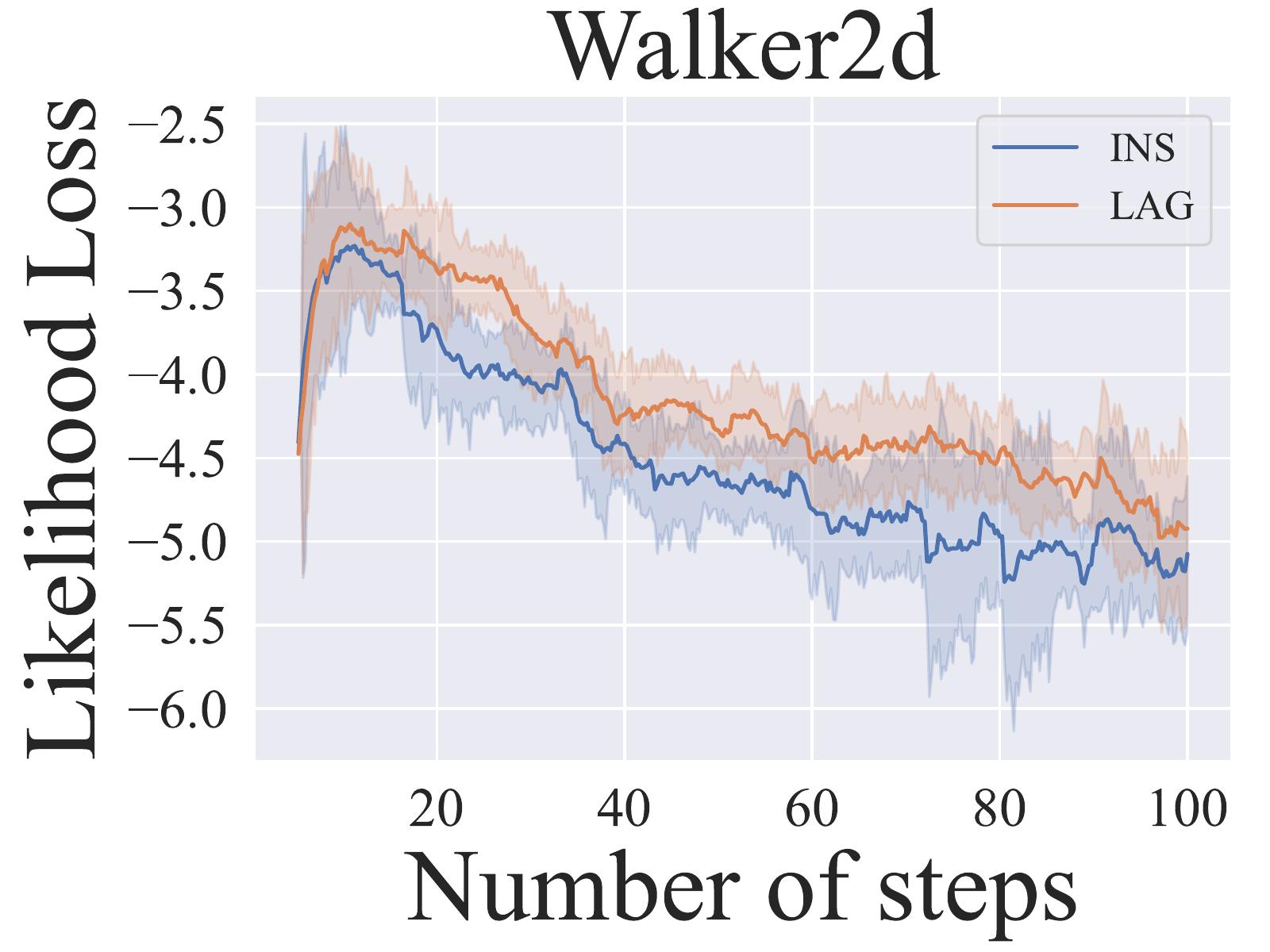} 
        \label{fig-return-inv}}
    \caption{The Comparisons of INS model and LAG model on four MuJoCo environments. \textbf{Top:} The comparison on policy returns. \textbf{Bottom:} The comparison on likelihood loss.}
    \label{fig-mujoco-return-likeli}
    \vspace{-2mm}
\end{figure*}

\begin{figure*}[th!]
\centering
\subfigure[]{
        \centering
        \includegraphics[width=0.3\linewidth]{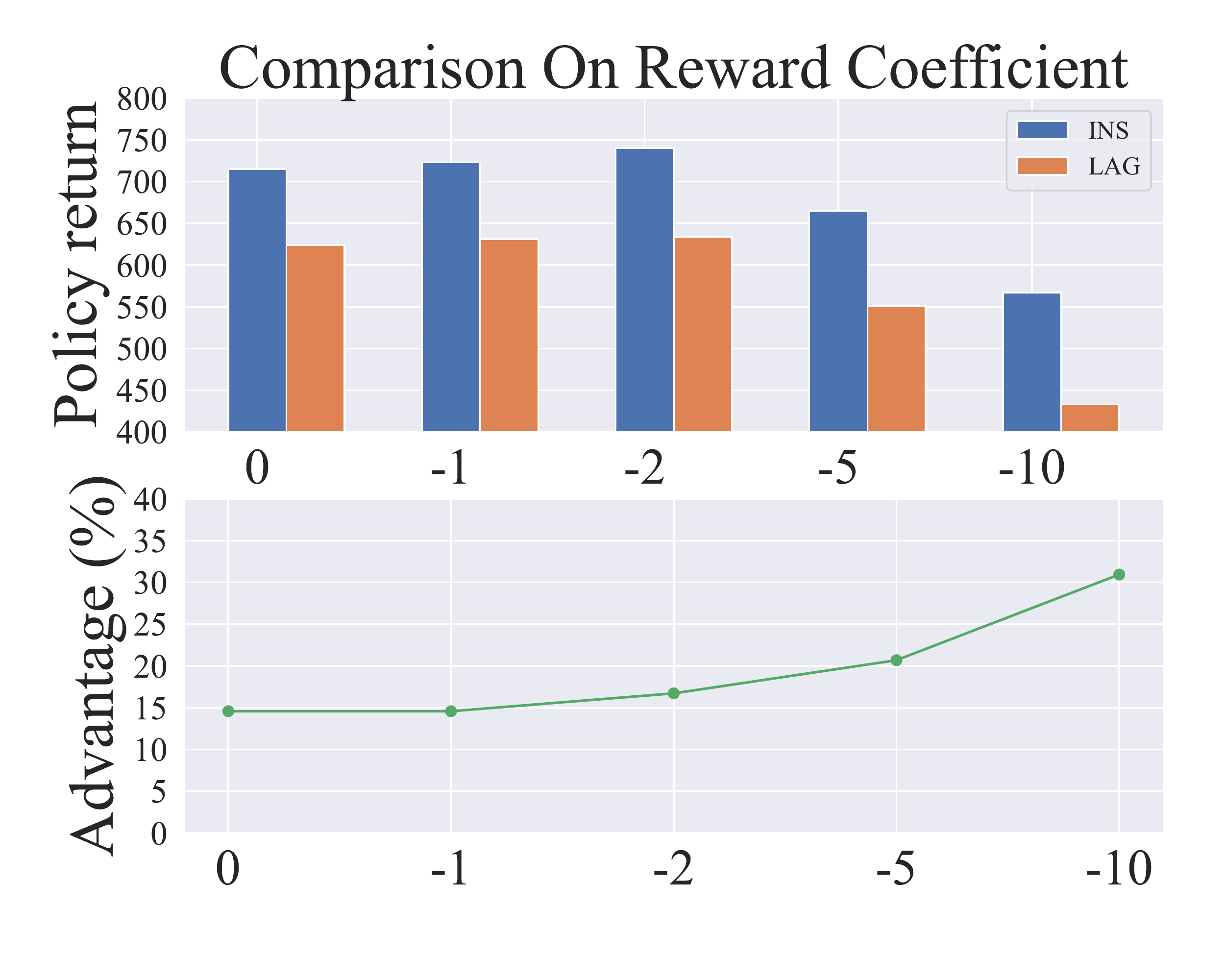} 
        \label{fig-invpen-reward-coef}}
\subfigure[]{
        \centering
        \includegraphics[width=0.3\linewidth]{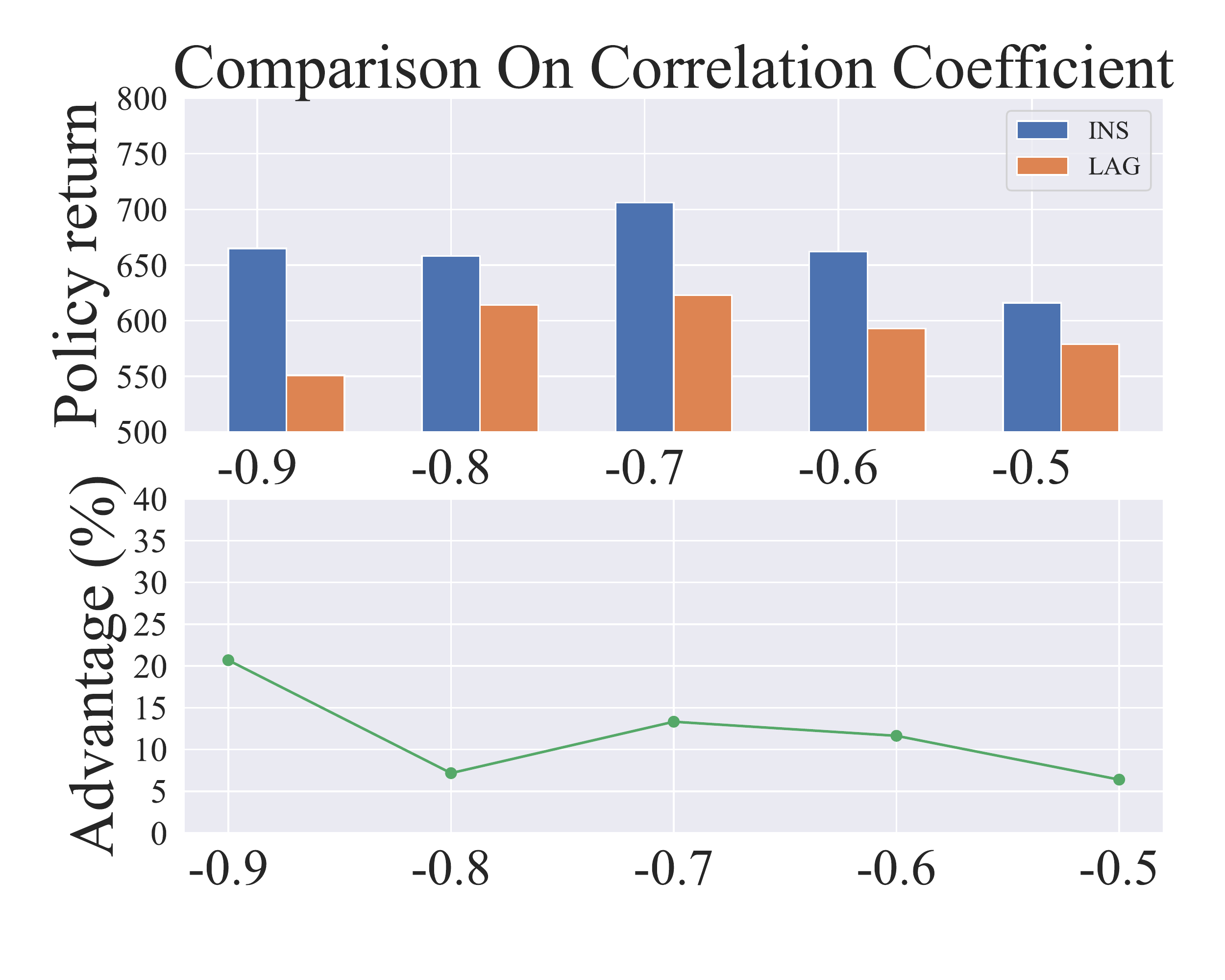} 
        \label{fig-invpen-covar-coef}}
\subfigure[]{
        \centering
        \includegraphics[width=0.3\linewidth]{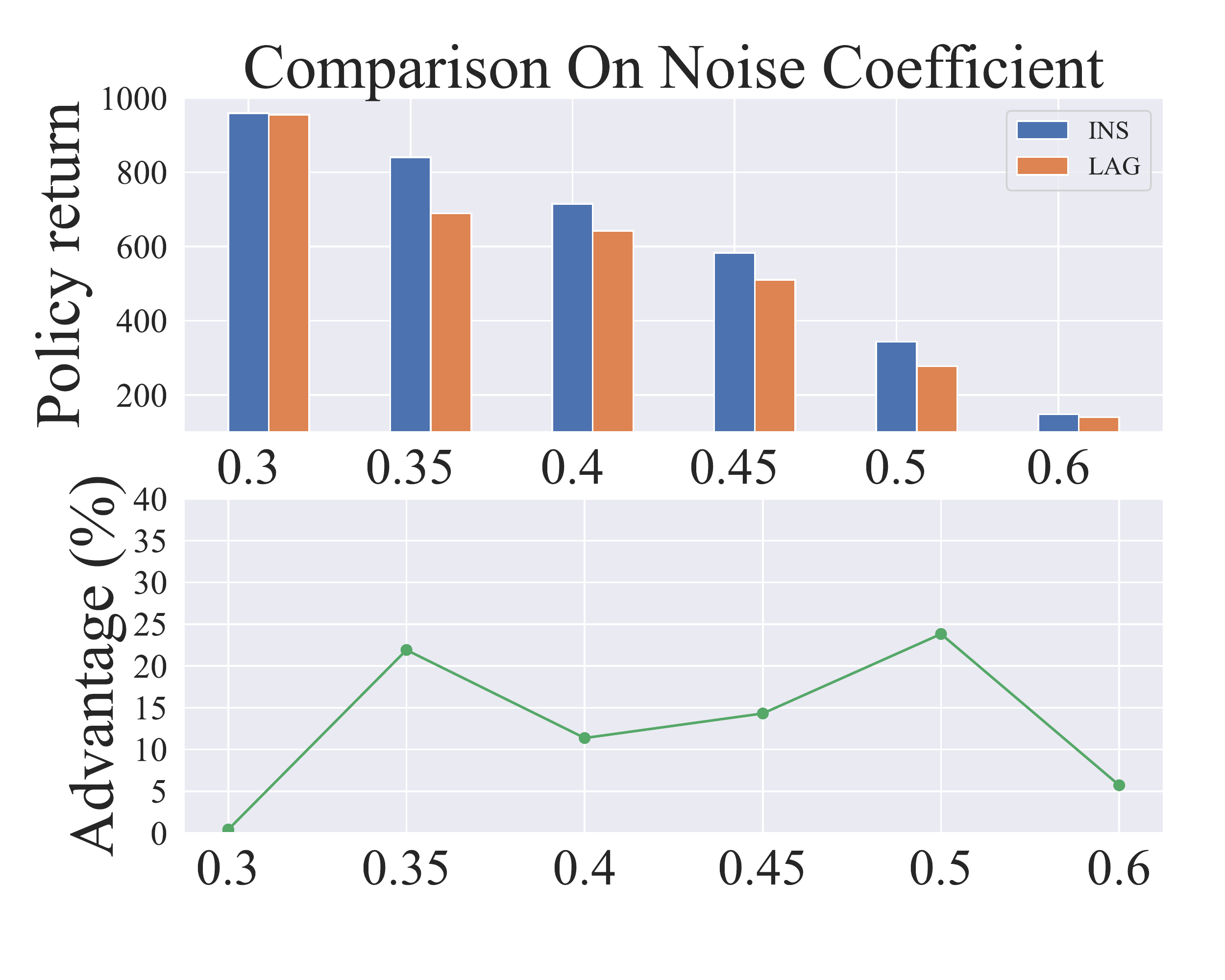} 
        \label{fig-invpen-noise-coef}}
     \vspace{-2mm}
\caption{The comparison of INS model and LAG model on varying coefficients.
\textbf{(a):} The policy returns of the two models with reward coefficients in $\{0,-1,-2,-5,-10\}$.
\textbf{(b): } The policy returns of the two models with correlation coefficients in $\{-0.9,-0.8,-0.7,-0.6,-0.5\}$.
\textbf{(c): }The policy returns of the two models with noise coefficients in $\{0.3,0.35,0.4,0.45,0.5,0.6\}$.
The lower part of each sub-figure is the policy return advantage of the INS model over the LAG model, where the value is computed by (INS - LAG)/LAG.}
\end{figure*}

We learn the optimal policy $\pi^*_{\hat{P}}$ from a lagged model and $\pi^*_{P}$ from the true environment, and visualize the policies in Figure~\ref{fig-visual}.
As circled in red in the figure, the two policies differ in that region, validating the optimality-inconsistency of the lagged model.
We also provide additional visualizations for different parameters of $\Delta t, \Delta v(A_i), \sigma_v, g(A_i)$, where the shape of the region changes with parameters while all the shapes match our theoretical results. 
Those visualization results can be found in the appendix.
We also evaluate the lagged model and the instantaneous model on policy returns for this task and the results in Table~\ref{table-visual-score} also confirm the optimality-inconsistency of the lagged model from the perspective of policy performance.

\subsection{Validation for the conditions of optimality-inconsistency}

To validate our derived conditions of optimality-inconsistency under general transition functions, we modify the CartPole environment by adding various reward items into the original reward function as shown in Table~\ref{table-reward-design}.
\begin{table}[h!]
\begin{tabular}{|c|c|}
\toprule
A &  $-s^2_{t,i}$. \\
\midrule
B & \makecell{$-(s_{t,i}+s_{t,j})^2$,\\ where noises of $s_{t,i},s_{t,j}$ are independent.}\\
\midrule
C &  \makecell{$-(s_{t,i}+s_{t,j})^2$,\\ where noises of $s_{t,i},s_{t,j}$ are dependent.}\\
\midrule
D &  \makecell{$-[(s_{t,i}+s_{t,j})^2+(s_{t,k}+s_{t,m})^2]$, \\where noises of $(s_{t,i},s_{t,j})$ and $(s_{t,k},s_{t,m})$ \\are both dependent.}\\
\midrule
E & \makecell{$-(s_{t,i}+s_{t,j})^2$, \\where noises of $s_{t,i},s_{t,j}$ are dependent \\and the correlation coefficient is $0.9$ \\(In other cases we set the coefficient as $0.5$).}\\
\bottomrule
\end{tabular}
\caption{The design of different reward functions.}
\label{table-reward-design}
\vspace{-2mm}
\end{table}

For rewards $A$ and $B$, we have $G_R\equiv 0$ while for $C$, $D$ and $E$, we have $G_R\nequiv 0$.
Here we make normalization for the above reward items so that the ranges of those items are all bounded in $[-1,0]$.
And these extra rewards will not influence the goal of CartPole, which is to stand up, because at the state of standing up, the extra rewards all equal $0$.

Results in Figure~\ref{fig-reward-bar} show that with reward items $A$ and $B$, the performance gap between using the two environment models is close to that with the original reward function while with reward items $C$, $D$, and $E$, the gap is significantly greater than that with the original reward function.
The difference in the performance gap supports our conclusion that under reward functions satisfying $G_R\nequiv 0$ and general transition functions, a lagged environment model is optimality-inconsistent.

\subsection{Evaluation of proposed method}

To validate that our proposed method successfully enables MBRL algorithms to consider instantaneous dependence, we evaluate our algorithm and the baseline in MuJoCo for four tasks (Inverted Pendulum, Hopper, HalfCheetah, Walker2d).
Since the four tasks have more complex transition functions than the CartPole, experiment results in this subsection can also support our claim in the last subsection.

We first introduce the details of our modification to MuJoCo.
We add dependent noises as follows to make it a stochastic RL environment for evaluation:
\begin{align*}
    \tilde{\textbf{s}}_{t+1} &= P(\textbf{s}_t,\textbf{a}_t),\\
    \textbf{s}_{t+1}&= \tilde{\textbf{s}}_{t+1} + \textbf{e}_t,
\end{align*}
where $P$ is the default transition function, $\textbf{e}\sim \mathcal{N}(\bm{0},\Sigma)$ is added dependent noises.
The noises are proportional to states, that is,
\begin{align*}
    Var(e_{t,i})=\Sigma_{i,i}=r_{noise}(\tilde{\textbf{s}}_{t+1,i} - \textbf{s}_{t,i}),
\end{align*}
where $r_{noise}$ is a coefficient to adjust the magnitude of noises.
We use proportional noises instead of noises with fixed variance 
because the difference between neighboring states ($f(\textbf{s}_t)-\textbf{s}_t$) distributes in a wide range,
in which using noises with fixed variance does not match the reality and often lead to terrible optimal policies. 
In addition, we set the correlation coefficient of dependent noises as $0.9$ or $-0.9$.

We also add extra reward items in MuJoCo according to our theoretical findings.
For dependent variable pair $(s_{t,i},s_{t,j})$, we add extra reward items in the form of $(s_{t,i}+s_{t,j})^2$ with normalization to keep the reward value in the range of $[0,1]$:
\begin{align*}
    \tilde{r}_t &= R(\textbf{s}_t,a_t),\\
    r_t &= \tilde{r}_t + r_{reward}\frac{1}{N_d}\sum^{N_d} \textrm{Norm}((s_{t,i}+s_{t,j})^2).
\end{align*}
where $N_d$ is the number of dependent variable pairs, $\textrm{Norm}(\cdot)\in[0,1]$ and $r_{reward}$ is a coefficient to adjust the weight of extra rewards.

In Figure~\ref{fig-mujoco-return-likeli}, we record likelihood losses and policy returns for our algorithm and the baseline in four tasks with proper noise coefficients and reward coefficients. 
The likelihood loss is the negative log of the likelihood probability that the real data is sampled from the learned environmental model. 
Specifically, for a learned model $\hat{P}(\bm{\mu},\Sigma)$ and collected data $\textbf{x}$ from the real environment, the likelihood loss is:
\begin{align*}
    &\textrm{Likelihood Loss}(\hat{P}(\bm{\mu},\Sigma),\textbf{x}) \\
    =& -2\log[\frac{1}{\sqrt{(2\pi)^n|\Sigma|}}e^{-\frac{1}{2}(\textbf{x}-\bm{\mu})^T\Sigma^{-1}(\textbf{x}-\bm{\mu})}]\\
    =& \log(|\Sigma|)+(\textbf{x}-\textbf{u})^T\Sigma^{-1}(\textbf{x}-\textbf{u})+n\log(2\pi)
\end{align*}
The results show that our algorithm achieves smaller likelihood loss and higher policy returns, which indicates that our model makes predictions with instantaneous dependence, and it also benefits policy learning.

\subsection{Hyper-parameters Sensitivity Study}

In the previous subsection, we evaluate our algorithm with fixed hyper-parameters, for example, we set the noise coefficient to $0.4$, the reward coefficient to $-5$, and the correlation coefficient to $-0.9$ in the Inverted Pendulum task.
In this subsection, we conduct experiments with varying hyper-parameters to demonstrate that our algorithm robustly outperforms the baseline.
We test one hyper-parameter at a time while keeping other hyper-parameters at default values.

The comparison in Figure~\ref{fig-invpen-reward-coef} demonstrates that our algorithm outperforms the baseline consistently with varying reward coefficients, and the advantages increase with higher coefficients.
The policy returns first increase at first and then decrease with increasing coefficient because the added rewards in the Inverted Pendulum help keep the pole standing, but too high reward coefficients make it difficult for the agent to obtain positive rewards and learn effectively.

The comparison in Figure~\ref{fig-invpen-covar-coef} shows that our algorithm outperforms the baseline with varying correlation coefficients.
And roughly speaking, the advantages increase with higher absolute values of the coefficient.

The comparison in Figure~\ref{fig-invpen-noise-coef} demonstrates that our algorithm outperforms the baseline with proper noise coefficients.
Both algorithms reach the upper bound score of the task with extremely small noises. 
And both algorithms fail to achieve good performance with extremely large noises.

\section{Conclusion}

To the best of our knowledge, this paper is the first to consider instantaneous dependence in MBRL, whereas existing works assume conditional independence for state variables given the past states.
To answer the questions that whether and when we need to consider instantaneous dependence, we consider the optimality-consistency of an environment model and prove that the optimality-consistency can only be guaranteed under specific and rare conditions.
We further derive sufficient conditions for optimality-inconsistency in certain situations.
 Additionally, we propose a simple method to enable existing algorithms to consider instantaneous dependence.
Through experiments, we (1) verify the optimality-inconsistency of a lagged model with visualization, (2) empirically validate derived conditions for optimality-inconsistency in certain situations, and (3) validate our proposed method successfully makes more accurate predictions and thus improves policy performance.
This paper takes a meaningful step in studying MBRL with instantaneous dependence.
An interesting future direction would be to investigate how to improve the efficiency and accuracy of learning instantaneous dependence in different RL settings, such as Offline RL.

\bibliography{icml2023}
\bibliographystyle{icml2023}

\newpage
\clearpage
\par
\newpage
\appendix
\onecolumn

\section{Theory}

\begin{lemma}
    Given an arbitrary stochastic transition function $P$ and corresponding $\hat{P}$, if there exists $\tilde{\textbf{s}}$ that is visited by the optimal policy $\pi^*_{P}$, such that 
    \begin{align*}
        &Q^*_{\hat{P}}(\tilde{\textbf{s}},\pi^*_{P}(\tilde{\textbf{s}}))<Q^*_{\hat{P}}(\tilde{\textbf{s}},\pi^*_{\hat{P}}(\tilde{\textbf{s}}))\\
        &Q^*_{P}(\tilde{\textbf{s}},\pi^*_{P}(\tilde{\textbf{s}}))>Q^*_{P}(\tilde{\textbf{s}},\pi^*_{\hat{P}}(\tilde{\textbf{s}})),
    \end{align*}
    we have that the policy return of $J(\pi^*_{P})>J(\pi^*_{\hat{P}})$, that is, $\pi^*_{\hat{P}}$ strictly underperforms the optimal policy $\pi^*_{P}$.
\end{lemma}

\begin{proof}
    Trivial.
    \begin{align*}
        J(\pi^*_{P}) =& \int p(\textbf{s})Q^*_{P}(\textbf{s},\pi^*_{P}(\textbf{s})) \\
        = & \int_{\mathcal{S}/\{\tilde{\textbf{s}}\}} p(\textbf{s})Q^*_{P}(\textbf{s},\pi^*_{P}(\textbf{s}))d\textbf{s}\\
        &+p(\tilde{\textbf{s}})Q^*_{P}(\tilde{\textbf{s}},\pi^*_{P}(\tilde{\textbf{s}}))\\
        \geq &\int_{\mathcal{S}/\{\tilde{\textbf{s}}\}} p(\textbf{s})Q^*_{P}(\textbf{s},\pi^*_{\hat{P}}(\textbf{s}))d\textbf{s}\\
        &+p(\tilde{\textbf{s}})Q^*_{P}(\tilde{\textbf{s}},\pi^*_{P}(\tilde{\textbf{s}}))\\
        > &\int_{\mathcal{S}/\{\tilde{\textbf{s}}\}} p(\textbf{s})Q^*_{P}(\textbf{s},\pi^*_{\hat{P}}(\textbf{s}))d\textbf{s}\\
        &+p(\tilde{\textbf{s}})Q^*_{P}(\tilde{\textbf{s}},\pi^*_{\hat{P}}(\tilde{\textbf{s}}))\\
        =&J(\pi^*_{\hat{P}})
    \end{align*}
\end{proof}

\begin{proof}
\begin{align*}
    &P(s_{t,i},s_{t,j}|\textbf{s}_{t-1},a_{t-1})
    =P(s_{t,i}|\textbf{s}_{t-1},a_{t-1})P(s_{t,j}|\textbf{s}_{t-1},a_{t-1},s_{t,i})
\end{align*}
According to the property of lagged model and the definition of instantaneous dependence, we have
\begin{align*}
    \hat{P}(s_{t,i}|\textbf{s}_{t-1},a_{t-1})&=P(s_{t,i}|\textbf{s}_{t-1},a_{t-1})\\
    \hat{P}(s_{t,j}|\textbf{s}_{t-1},a_{t-1})&=P(s_{t,j}|\textbf{s}_{t-1},a_{t-1})\\
    &\neq P(s_{t,j}|\textbf{s}_{t-1},a_{t-1},s_{t,i}).
\end{align*}
Therefore, 
\begin{align*}
    \hat{P}(s_{t,i},s_{t,j}|\textbf{s}_{t-1},a_{t-1})\neq P(s_{t,i},s_{t,j}|\textbf{s}_{t-1},a_{t-1}).
\end{align*}
\end{proof}

\begin{corollary}
\label{corollary-integral-app}
For any stochastic transition function $P$ and corresponding $\hat{P}$, any $(\textbf{s}_{t-1},a_{t-1})$, we have four corollaries as follows:

(1) For any constant $C$, we have
\begin{align*}
    \int [P(\textbf{s}_t|\textbf{s}_{t-1},a_{t-1})-\hat{P}(\textbf{s}_t|\textbf{s}_{t-1},a_{t-1})]Cd\textbf{s}_t=0.
\end{align*}
(2) For any $s_{t,i}\in\textbf{s}_t$ and an arbitrary function $f(s_{t,i})$, we have
\begin{align*}
    &\int [P(\textbf{s}_t|\textbf{s}_{t-1},a_k)-\hat{P}(\textbf{s}_t|\textbf{s}_{t-1},a_k)]f(s_{t,i})d\textbf{s}_t=0.   
\end{align*}
(3) For any independent variable pair $(s_{t,i},s_{t,j})$ and an arbitrary function $g(s_{t,i},s_{t,j})$, we have
\begin{align*}
    &\int [P(\textbf{s}_t|\textbf{s}_{t-1},a_k)-\hat{P}(\textbf{s}_t|\textbf{s}_{t-1},a_k)]g(s_{t,i},s_{t,j})d\textbf{s}_t=0.
\end{align*}
(4) If there exists instantaneous dependence $(s_{t,i},s_{t,j})$ in $P$, then there exists a function $g(s_{t,i},s_{t,j})$ that cannot be divided into $g_1(s_{t,i})+g_2(s_{t,j})$ such that
\begin{align*}
    &\int [P(\textbf{s}_t|\textbf{s}_{t-1},a_k)-\hat{P}(\textbf{s}_t|\textbf{s}_{t-1},a_k)]g(s_{t,i},s_{t,j})d\textbf{s}_t\neq 0.
\end{align*}
\end{corollary}

 \begin{proof}
For any constant $C$,
\begin{align*}
&\int [P(\textbf{s}_t|\textbf{s}_{t-1},a_{t-1})-\hat{P}(\textbf{s}_t|\textbf{s}_{t-1},a_{t-1})]Cd\textbf{s}_t\\
    &=C\int [P(\textbf{s}_t|\textbf{s}_{t-1},a_{t-1})-\hat{P}(\textbf{s}_t|\textbf{s}_{t-1},a_{t-1})]d\textbf{s}_t\\
    &=0
\end{align*}

For any $s_{t,i}\in\textbf{s}_t$ and an arbitrary function $f(s_{t,i})$, we have
\begin{align*}
    &\int [P(\textbf{s}_t|\textbf{s}_{t-1},a_k)-\hat{P}(\textbf{s}_t|\textbf{s}_{t-1},a_k)]f(s_{t,i})d\textbf{s}_t\\
    =&\int [P(s_{t,i}|\textbf{s}_{t-1},a_k)-\hat{P}(s_{t,i}|\textbf{s}_{t-1},a_k)] f(s_{t,i})ds_{t,i}\\
    =&0   
\end{align*}

For any independent variable pair $(s_{t,i},s_{t,j})$ and an arbitrary function $g(s_{t,i},s_{t,j})$, we have
\begin{align*}
    &\int [P(\textbf{s}_t|\textbf{s}_{t-1},a_{t-1})-\hat{P}(\textbf{s}_t|\textbf{s}_{t-1},a_{t-1})]g(s_{t,i},s_{t,j})d\textbf{s}_t\\
    =&\int\int [P(s_{t,i},s_{t,j}|\textbf{s}_{t-1},a_{t-1})-\hat{P}(s_{t,i},s_{t,j}|\textbf{s}_{t-1},a_{t-1})]\\
    &g(s_{t,i},s_{t,j})ds_{t,i}ds_{t,j}\\
    =&\int\int [P(s_{t,i}|\textbf{s}_{t-1},a_{t-1})P(s_{t,j}|\textbf{s}_{t-1},a_{t-1})\\
    &-\hat{P}(s_{t,i}|\textbf{s}_{t-1},a_{t-1})\hat{P}(s_{t,j}|\textbf{s}_{t-1},a_{t-1})]
    g(s_{t,i},s_{t,j})ds_{t,i}ds_{t,j}\\
    =&\int\int P(s_{t,i}|\textbf{s}_{t-1},a_{t-1})P(s_{t,j}|\textbf{s}_{t-1},a_{t-1})g(s_{t,i},s_{t,j})ds_{t,i}ds_{t,j}\\
    &-\int\int \hat{P}(s_{t,i}|\textbf{s}_{t-1},a_{t-1})\hat{P}(s_{t,j}|\textbf{s}_{t-1},a_{t-1})g(s_{t,i},s_{t,j})ds_{t,i}ds_{t,j}\\
    =&\int[\int P(s_{t,i}|\textbf{s}_{t-1},a_{t-1})g(s_{t,i},s_{t,j})ds_{t,i}]P(s_{t,j}|\textbf{s}_{t-1},a_{t-1})ds_{t,j}\\
    &-\int[\int \hat{P}(s_{t,i}|\textbf{s}_{t-1},a_{t-1})g(s_{t,i},s_{t,j})ds_{t,i}]\hat{P}(s_{t,j}|\textbf{s}_{t-1},a_{t-1})ds_{t,j}\\
    =&\int[\int \hat{P}(s_{t,i}|\textbf{s}_{t-1},a_{t-1})g(s_{t,i},s_{t,j})ds_{t,i}]P(s_{t,j}|\textbf{s}_{t-1},a_{t-1})ds_{t,j}\\
    &-\int[\int \hat{P}(s_{t,i}|\textbf{s}_{t-1},a_{t-1})g(s_{t,i},s_{t,j})ds_{t,i}]\hat{P}(s_{t,j}|\textbf{s}_{t-1},a_{t-1})ds_{t,j}\\
    & (\textrm{Since $P(s_{t,i}|\textbf{s}_{t-1},a_{t-1})= \hat{P}(s_{t,i}|\textbf{s}_{t-1},a_{t-1})$}) \\
    =&\int[\tilde{g}(s_{t,j}|\textbf{s}_{t-1},a_{t-1})]P(s_{t,j}|\textbf{s}_{t-1},a_{t-1})ds_{t,j}\\
    &-\int[\tilde{g}(s_{t,j}|\textbf{s}_{t-1},a_{t-1})]\hat{P}(s_{t,j}|\textbf{s}_{t-1},a_{t-1})ds_{t,j}\\
    =&0,
\end{align*}
where $\tilde{g}(s_{t,j}|\textbf{s}_{t-1},a_{t-1})=\int \hat{P}(s_{t,i}|\textbf{s}_{t-1},a_{t-1})g(s_{t,i},s_{t,j})ds_{t,i}$.

If there exists instantaneous dependence $(s_{t,i},s_{t,j})$ in $P$, then there exists a function $g(s_{t,i},s_{t,j})$ that cannot be divided into $g_1(s_{t,i})+g_2(s_{t,j})$ such that
\begin{align*}
    &\int [P(\textbf{s}_t|\textbf{s}_{t-1},a_{t-1})-\hat{P}(\textbf{s}_t|\textbf{s}_{t-1},a_{t-1})]g(s_{t,i},s_{t,j})d\textbf{s}_t\\
    =&\int\int [P(s_{t,i},s_{t,j}|\textbf{s}_{t-1},a_{t-1})-\hat{P}(s_{t,i},s_{t,j}|\textbf{s}_{t-1},a_{t-1})]
    g(s_{t,i},s_{t,j})ds_{t,i}ds_{t,j}\\
    \neq&0.
\end{align*}
\end{proof}

\begin{theorem}
    \label{theorem-GV*-app}
    For any $(\textbf{s}_{t-1},a_{t-1})$, any deterministic function $F(\textbf{s}_t)$ with the partition in Formula~\ref{formula-V-partition} and corresponding $G_{F}(\textbf{s}_t)$, any stochastic function $P$ and corresponding $\hat{P}$, we have
    \begin{align}
    \label{formula-GV*}
         &\int [P(\textbf{s}_t|\textbf{s}_{t-1},a_{t-1})-\hat{P}(\textbf{s}_t|\textbf{s}_{t-1},a_{t-1})]F(\textbf{s}_t) d\textbf{s}_t\nonumber \\  
    =&\int [P(\textbf{s}_t|\textbf{s}_{t-1},a_{t-1})-\hat{P}(\textbf{s}_t|\textbf{s}_{t-1},a_{t-1})]G_{F}(\textbf{s}_t) d\textbf{s}_t.
    \end{align}    
\end{theorem}

\begin{proof}
    It is trivial based on Corollary~\ref{corollary-integral}.
\end{proof}

\begin{lemma}
\label{lemma-multi-k}
    If $G_R(\textbf{s}_{t})\equiv 0$, and $P$ is $G-invariant$, then $\forall t\in\mathbb{N}^+$, we have
    \begin{align*}
        Q^*_{P}(\textbf{s}_{t},a_{t})=Q^*_{\hat{P}}(\textbf{s}_{t},a_{t})
    \end{align*}
\end{lemma}

\begin{proof}
First, when $t=T-1$, we have
\begin{align*}
    Q^*_{P}(\textbf{s}_{T-1},a) &= R(\textbf{s}_{T-1},a) = Q^*_{\hat{P}}(\textbf{s}_{T-1},a), \ \forall a\\
\end{align*}
Therefore, we can get
\begin{align*}
    \pi^*_{P}(\textbf{s}_{T-1})=&\arg\max_{a}Q^*_{P}(\textbf{s}_{T-1},a)\\
    =&\arg\max_{a}Q^*_{\hat{P}}(\textbf{s}_{T-1},a)\\
    =&\pi^*_{\hat{P}}(\textbf{s}_{T-1})\\
    V^*_{P}(\textbf{s}_{T-1}) 
    =& Q^*_P(\textbf{s}_{T-1}, \pi^*_{P}(\textbf{s}_{T-1}))\\
    =&Q^*_P(\textbf{s}_{T-1}, \pi^*_{\hat{P}}(\textbf{s}_{T-1}))\\
    =&Q^*_{\hat{P}}(\textbf{s}_{T-1}, \pi^*_{\hat{P}}(\textbf{s}_{T-1}))\\
    =&V^*_{\hat{P}}(\textbf{s}_{T-1}).
\end{align*}

Since $G_R(\textbf{s}_{T-1})\equiv 0$, we have $G_{V^*_{P}}(\textbf{s}_{T-1})\equiv 0$.

Suppose that for $t=T-i,i = 1,\cdots, T$, we have
\begin{align*}
 Q^*_{P}(\textbf{s}_{T-i},a) =& Q^*_{\hat{P}}(\textbf{s}_{T-i},a), \ \forall a\\
\pi^*_{P}(\textbf{s}_{T-i})=&\pi^*_{\hat{P}}(\textbf{s}_{T-i})\\
V^*_{P}(\textbf{s}_{T-i}) =&V^*_{\hat{P}}(\textbf{s}_{T-i})\\
G_{V^*_{P}}(\textbf{s}_{T-i})\equiv& 0,
\end{align*}
and then we can get that 
\begin{align*}
    &Q^*_{P} (\textbf{s}_{T-i-1},a) \\
    =& R(\textbf{s}_{T-i-1},a) + \int P(\textbf{s}_{T-i}|\textbf{s}_{T-i-1},a)V^*_{P}(\textbf{s}_{T-i})d\textbf{s}_{T-i}\\
    =& R(\textbf{s}_{T-i-1},a) + \int \hat{P}(\textbf{s}_{T-i}|\textbf{s}_{T-i-1},a)V^*_{P}(\textbf{s}_{T-i})d\textbf{s}_{T-i}\\
    =& R(\textbf{s}_{T-i-1},a) + \int \hat{P}(\textbf{s}_{T-i}|\textbf{s}_{T-i-1},a)V^*_{\hat{P}}(\textbf{s}_{T-i})d\textbf{s}_{T-i}\\
    =&Q^*_{\hat{P}} (\textbf{s}_{T-i-1},a).
\end{align*}
And we also have
\begin{align*}
Q^*_{P}(\textbf{s}_{T-i-1},a) = &Q^*_{\hat{P}}(\textbf{s}_{T-i-1},a), \ \forall a\\
    \pi^*_{P}(\textbf{s}_{T-i-1})=&\pi^*_{\hat{P}}(\textbf{s}_{T-i-1})\\
    V^*_{P}(\textbf{s}_{T-i-1})=&V^*_{\hat{P}}(\textbf{s}_{T-i-1})\\
    G_{V^*_{P}}(\textbf{s}_{T-i-1})\equiv& 0.
\end{align*}
By induction, for $t=1, \cdots, T-1$, we have $Q^*_{P} (\textbf{s}_{t},a_t) = Q^*_{\hat{P}} (\textbf{s}_{t},a_t)$.
\end{proof}

\begin{theorem}
\label{multi-app}
     In a RL environment, if for any policy $\pi$, the stochastic transition function $P(\textbf{s}_t|\textbf{s}_{t-1},\pi(\textbf{s}_{t-1}))$ is $G-invariant$ and the reward function $R(\textbf{s}_{t},\pi(\textbf{s}_{t}))$ as a function of $\textbf{s}_{t}$ satisfies that $G_R(\textbf{s}_{t})\equiv 0$, then we have $\beta = 0$ for any $(\textbf{s}_t,a_0,a_1)$.
\end{theorem}

\begin{proof}
Based on the Lemma~\ref{lemma-multi-k}, we have 
\begin{align*}
Q^*_{P}(\textbf{s}_{t-1},a_0)&=Q^*_{\hat{P}}(\textbf{s}_{t-1},a_0)\\
Q^*_{P}(\textbf{s}_{t-1},a_1)&=Q^*_{\hat{P}}(\textbf{s}_{t-1},a_1).
\end{align*}
Therefore we have $\beta = 0$.
\end{proof}

\begin{lemma}
    In a single-step RL, for any stochastic transition function $P$, any $(\textbf{s}_{t-1},a_0,a_1)$, any reward function $R$, any function $f$ and state $s_{t,i}$,
    we have 
    \begin{align*}
        &\beta(R(\textbf{s}_{t-1},a_{t-1},\textbf{s}_t))=\beta(R(\textbf{s}_{t-1},a_{t-1},\textbf{s}_t)+f(s_{t,i}))\\
        &\beta(R)=-\beta(-R)
    \end{align*}
\end{lemma}
\begin{proof}
    It is trivial based on Corollary~\ref{corollary-integral}.
\end{proof}

\begin{theorem}
\label{single-DR}
    In a single-step RL, for any stochastic transition function $P$ with instantaneous dependence, any $(\textbf{s}_{t-1},a_0,a_1)$, we have
    $\mathcal{D}_R \neq \emptyset$.
\end{theorem}
\begin{proof}
proof by contradiction:

If there exist $P,\textbf{s}_{t-1},a_0,a_1$ such that $\mathcal{D}_R=\emptyset$,
then there must exists a reward function $R_0$ that $R_0\notin \mathcal{D}_R$, which means $\alpha(R_0)\leq 0$ or $\alpha(R_0)\geq -\beta(R_0)$.
Since there is instantaneous dependence in $P$, we assume $\beta(R_0)\neq 0$ WLOG.
In the following, we will design a new reward function $\tilde{R}$ such that $\tilde{R}\subset \mathcal{D}_R$.

First, if $\beta(R_0)<0$, we let $R'_0=-R_0$ and have $\beta(R'_0)=-\beta(R_0)>0$.
Therefore we assume $\beta(R_0)>0$ WLOG.

Second, since $P(\textbf{s}_t|\textbf{s}_{t-1},a_0)\nequiv P(\textbf{s}_t|\textbf{s}_{t-1},a_1)$, there exists a function $f(s_{t,i})$ such that 
\begin{align*}
    &\int [P(\textbf{s}_t|\textbf{s}_{t-1},a_0)- P(\textbf{s}_t|\textbf{s}_{t-1},a_1)]f(s_{t,i})d\textbf{s}_{t}
    \neq0.
\end{align*}
We denote 
$$K=\int [P(\textbf{s}_t|\textbf{s}_{t-1},a_0)- P(\textbf{s}_t|\textbf{s}_{t-1},a_1)]f(s_{t,i})d\textbf{s}_{t}$$ 
and assume $K>0$ WLOG.
We define 
\begin{align*}
    R_1(\textbf{s}_{t-1},a_{t-1},\textbf{s}_t) = R_0(\textbf{s}_{t-1},a_{t-1},\textbf{s}_t)+xf(s_{t,i}),
\end{align*}
and have that
\begin{align*}
    &\alpha(R_1)=\alpha(R_0)+xK\\
    &\beta(R_1)=\beta(R_0).
\end{align*}
We choose $x\in (-\frac{\alpha(R_0)}{K}, -\frac{\alpha(R_0)+\beta(R_0)}{K})$ and have
\begin{align*}
&\left\{
\begin{array}{cc}
     &\alpha(R_1)>0\\
    &\alpha(R_1)+\beta(R_1)<0\\
\end{array}
\right.\\
    \Rightarrow& R_1\in \mathcal{D}_R.
\end{align*}
Therefore, $\mathcal{D}_R\neq \emptyset$.

Since 
\begin{align*}
    J(\pi^*_{P})-J(\pi^*_{\hat{P}})\geq \alpha,
\end{align*}
we can find that the performance gap of the two models is unbounded.
\end{proof}

\begin{theorem}
\label{multi-DR}
    For any stochastic and $G-variant$ transition function $P$ with instantaneous dependence, any $(\textbf{s}_{t-1},a_0,a_1)$, we have
    $\mathcal{D}_R \neq \emptyset$.
\end{theorem}
\begin{proof}
    Similarly, we can add $f(s_{t,i})$ to $R(\textbf{s}_{t-1},a_t,\textbf{s}_t)$ with which $\beta$ is invariant and $\alpha$ changes.
\end{proof}

\section{Experiment}
\subsection{Experiment details}
Suppose that a car moves in a 1-D space with observations $\textbf{s}_t=(v_t,p_t)$ and actions $a_{t}$ where $p_t$ represents the position, $v_t$ represents the velocity and $a_{t}\subset\{a_0,a_1\}$ represents different acceleration process.
The transition function is that 
\begin{align*}
v_t&=v_{t-1}+(-\frac{p_{t-1}}{|p_{t-1}|})\Delta v(a_{t-1})+g(a_{t-1})\epsilon_v \\
    p_t&=p_{t-1}+v_t\Delta t +\epsilon_p,
\end{align*}
where $\Delta v(a_{t-1})$ represents the velocity change and $g(a_{t-1})$ represents the noise magnitude under acceleration $a_{t-1}$.
$(-\frac{p_{t-1}}{|p_{t-1}|})$ is the coefficient of $\Delta v(a_{t-1})$ to make sure that actions always help the position to approximate the origin.  
We set the reward function as $R(p_{t},v_t,p_{t-1},v_{t-1})=p_tv_t-p_{t-1}v_{t-1}$, which implies that $V(p_t,v_t)=p_tv_t$, and then we can extend the function $0<\alpha<-\beta$.

which can be extended to:
    \begin{align*}
    \frac{p_{t-1}}{\Delta t}+2v_{t-1}<&-(\Delta v(a_1)+\Delta v(a_0))\\
    \frac{p_{t-1}}{\Delta t}+2v_{t-1}>&-(\Delta v(a_1)+\Delta v(a_0))\\
    &-\sigma_v^2\frac{g^2(a_1)-g^2(a_0)}{\Delta v(a_1)-\Delta v(a_0)}.
\end{align*}

\subsection{Another example for 1-D Driving}
Let us give a simple example for above analysis.
Suppose that a car moves in a 1-D space with observations $\textbf{s}_t=(v_t,p_t)$ and actions $a_{t}$ where $p_t$ represents the position, $v_t$ represents the velocity and $a_{t}\subset\{a_0,a_1\}$ represents different acceleration process.
There exists instantaneous effect that 
\begin{align*}
v_t&=v_{t-1}+\Delta v(a_{t-1})+g(a_{t-1})\epsilon_v \\
    p_t&=p_{t-1}+v_t\Delta t +\epsilon_p,
\end{align*}
where $\Delta v(a_{t-1})$ represents the velocity change and $g(a_{t-1})$ represents the noise magnitude under acceleration $a_{t-1}$.

\begin{lemma}
Suppose that $\textbf{x} \sim \mathcal{N}(\mu,\Sigma)$, we have
\begin{align*}
    \int_\textbf{x} p(\textbf{x})\textbf{x}d\textbf{x}&=\mu, \\
    \int_\textbf{x} p(\textbf{x})\textbf{x}^TA\textbf{x}d\textbf{x}&=\mu^TA\mu + Tr(A\Sigma),
\end{align*}
where $Tr(\cdot)$ is the trace of a matrix.
\end{lemma}

Based on the lemma above, we suppose $V(\textbf{s}_t)=\textbf{s}_t^T A\textbf{s}_t$ and have that 
\begin{align*}
&\left\{
\begin{array}{c}
    Q^*_{P}(\tilde{\textbf{s}},a_0)-Q^*_{P}(\tilde{\textbf{s}},a_1)>0\\
       Q^*_{\hat{P}}(\tilde{\textbf{s}},a_0)-Q^*_{\hat{P}}(\tilde{\textbf{s}},a_1)<0\\
\end{array}
\right.\\
   \Leftrightarrow &   
   \left\{
   \begin{array}{c}
    \mu_1^TA\mu_1 + Tr(A\Sigma_1) - [\mu_0^TA\mu_0 + Tr(A\Sigma_0)]>0\\
    \mu_1^TA\mu_1 + Tr(A\hat{\Sigma}_1)- [\mu_0^TA\mu_0 + Tr(A\hat{\Sigma}_0)]<0,
\end{array}
\right.
\end{align*}

Assume that $A$ of the optimal value function $V^*$ satisfies $A_{1,1}=A_{2,2}=0$, $A_{1,2}+A_{2,1}>0$ and $\Delta v(a_1)-\Delta v(a_0)>0$, $g^2(a_1)-g^2(a_0)>0$, the above conditions are equivalent to
\begin{align*}
    &\mu_1^TA\mu_1 + Tr(A\hat{\Sigma}_1)- [\mu_0^TA\mu_0 + Tr(A\hat{\Sigma}_0)]<0\\
\Leftrightarrow&(A_{1,2}+A_{2,1})[(v_{t-1}+\Delta v(a_1))(p_{t-1}+v_{t-1}\Delta t+\Delta v(a_1)\Delta t)\\
& -(v_{t-1}+\Delta v(a_0))(p_{t-1}+v_{t-1}\Delta t+\Delta v(a_0)\Delta t)]<0\\
\Leftrightarrow&(A_{1,2}+A_{2,1})(\Delta v(a_1)-\Delta v(a_0))\Delta t\\
&[\frac{p_{t-1}}{\Delta t}+2v_{t-1}+\Delta v(a_1)+\Delta v(a_0)]<0.
\end{align*}
Similarly, 
\begin{align*}
    &\mu_1^TA\mu_1 + Tr(A\Sigma_1) - [\mu_0^TA\mu_0 + Tr(A\Sigma_0)]>0\\
\Leftrightarrow&(A_{1,2}+A_{2,1})[(v_{t-1}+\Delta v(a_1))(p_{t-1}+v_{t-1}\Delta t+\Delta v(a_1)\Delta t)\\
& -(v_{t-1}+\Delta v(a_0))(p_{t-1}+v_{t-1}\Delta t+\Delta v(a_0)\Delta t)\\
&+g^2(a_1)\Delta t\sigma_v^2-g^2(a_0)\Delta t\sigma_v^2]>0\\
\Leftrightarrow&(A_{1,2}+A_{2,1})(\Delta v(a_1)-\Delta v(a_0))\Delta t\\
&[\frac{p_{t-1}}{\Delta t}+2v_{t-1}+\Delta v(a_1)+\Delta v(a_0)\\
&+\sigma_v^2\frac{g^2(a_1)-g^2(a_0)}{\Delta v(a_1)-\Delta v(a_0)}]>0 \\
\end{align*}
In summary, for states $(p_{t-1},v_{t-1})$ in the area that
\begin{align*}
    \frac{p_{t-1}}{\Delta t}+2v_{t-1}<&-(\Delta v(a_1)+\Delta v(a_0))\\
    \frac{p_{t-1}}{\Delta t}+2v_{t-1}>&-(\Delta v(a_1)+\Delta v(a_0))\\
    &-\sigma_v^2\frac{g^2(a_1)-g^2(a_0)}{\Delta v(a_1)-\Delta v(a_0)},
\end{align*}
the optimal policy learned from a lagged model is different from the ground truth optimal policy.
We learn the optimal policy $\pi^*_{\hat{P}}$ from a lagged model and $\pi^*_{P}$ from the true environment, and plot the state-action region where the two policies choose different actions in Figure~\ref{Fig-driving-action-all}.

\begin{figure}[ht!]
    \centering
    \includegraphics[width=\linewidth]{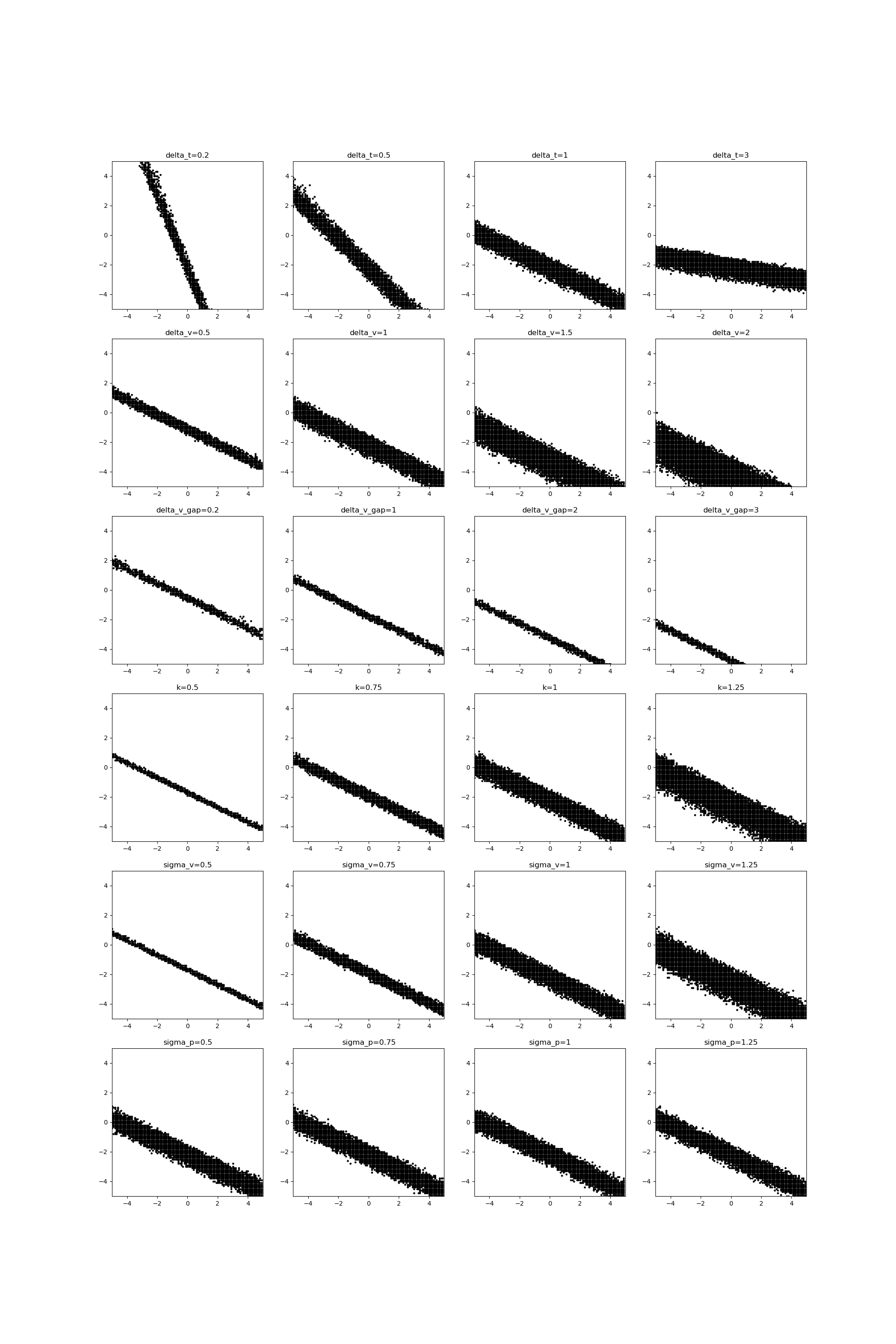}
    \caption{The collection of regions with different parameters.
    Each row represents one parameter and each column represents one value of the parameter.}
    \label{Fig-driving-action-all}
\end{figure}

From the definitions, we can easily find the task is to approximate the coordinate origin with small velocity and $a_1$ represents coarse tune while $a_0$ represents fine tune for the velocity of the car.

\end{document}